\documentclass{article}

\usepackage{microtype}
\usepackage{graphicx}
\usepackage{subfigure}
\usepackage{booktabs} \usepackage{pifont}
\usepackage{enumitem}
\usepackage{hyperref}

\usepackage[accepted]{icml2024}

\usepackage{amsmath}
\usepackage{amssymb}
\usepackage{mathtools}
\usepackage{amsthm}
\usepackage{bm}
\usepackage{array}
\usepackage{multirow}

\usepackage[capitalize,noabbrev]{cleveref}

\usepackage[createShortEnv,conf={no link to proof, text link},commandRef=autoref]{proof-at-the-end}

\theoremstyle{plain}
\newtheorem{theorem}{Theorem}[section]
\newtheorem{proposition}[theorem]{Proposition}
\newtheorem{lemma}[theorem]{Lemma}

\theoremstyle{definition}
\newtheorem{definition}[theorem]{Definition}

\newtheorem{remark}[theorem]{Remark}

\usepackage[textsize=tiny]{todonotes}

\icmltitlerunning{Switched Flow Matching}

\begin{document}
  
  \twocolumn[
  \icmltitle{Switched Flow Matching: Eliminating Singularities via Switching ODEs}
  
    \begin{icmlauthorlist}
        \icmlauthor{Qunxi Zhu}{IICS}
        \icmlauthor{Wei Lin}{IICS,MS,MOE,AIlab}
    \end{icmlauthorlist}

    \icmlaffiliation{IICS}{Research Institute of Intelligent Complex Systems, Fudan University, China.}
    \icmlaffiliation{MS}{School of Mathematical Sciences, LMNS, and SCMS, Fudan University, China.}
    \icmlaffiliation{MOE}{State Key Laboratory of Medical Neurobiology and MOE Frontiers Center for Brain Science, Institutes of Brain Science, Fudan University, China.}
    \icmlaffiliation{AIlab}{Shanghai Artificial Intelligence Laboratory, China}
    
    \icmlcorrespondingauthor{Qunxi Zhu}{qxzhu16@fudan.edu.cn}
    
    \icmlkeywords{Machine Learning, ICML}
    
    \vskip 0.3in
    ]
    
    \printAffiliationsAndNotice{}
    \begin{abstract}
      Continuous-time generative models, such as Flow Matching (FM), construct probability paths to transport between one distribution and another through the simulation-free learning of the neural ordinary differential equations (ODEs). During inference, however, the learned model often requires multiple neural network evaluations to accurately integrate the flow, resulting in a slow sampling speed. We attribute the reason to the inherent (joint) heterogeneity of source and/or target distributions, namely the singularity problem, which poses challenges for training the neural ODEs effectively. To address this issue, we propose a more general framework, termed Switched FM (SFM),  that eliminates singularities via switching ODEs, as opposed to using a uniform ODE in FM. Importantly, we theoretically show that FM cannot transport between two simple distributions due to the existence and uniqueness of initial value problems of ODEs, while these limitations can be well tackled by SFM. From an orthogonal perspective, our framework can seamlessly integrate with the existing advanced techniques, such as minibatch optimal transport, to further enhance the straightness of the flow, yielding a more efficient sampling process with reduced costs. We demonstrate the effectiveness of the newly proposed SFM through several numerical examples.
    \end{abstract}
    
    \section{Introduction}
    Generative modeling is a fundamental task in the machine learning and data science communities, whose primary objective is to transform samples from one (empirical) probability distribution to another through a learnable transformation. Over the years, several methods have been extensively proposed for generative modeling, including generative adversarial networks (GAN) \cite{goodfellow2014generative},  variational autoencoders (VAE) \cite{kingma2013auto, rezende2014stochastic}, energy-based models \cite{teh2003energy, lecun2006tutorial, du2019implicit,song2021train},  normalizing flow models \cite{dinh2014nice, dinh2016density, rezende2015variational}, and autoregressive models \cite{germain2015made, van2016pixel, van2016conditional,oord2016wavenet}.
    
    Despite their successes across various domains, these models have some limitations. For instance, training GANs can be challenging because of several major issues, including mode collapse \cite{goodfellow2014generative, metz2016unrolled}, vanishing gradient \cite{arjovsky2017wasserstein, weng2019gan}, and unstable convergence \cite{arjovsky2017towards, farnia2020gans}. VAE and energy-based models employ surrogate losses to aid in successful training via utilizing the evidence lower bound (with the parameterization trick) \cite{kingma2013auto} and contrastive divergence \cite{hinton2002training}, respectively. Normalizing flow \cite{dinh2014nice, dinh2016density, rezende2015variational} and autoregressive models \cite{germain2015made, van2016pixel, van2016conditional,oord2016wavenet} often impose architectural constraints to build a normalized probability model.
    
    Diffusion models \cite{sohl2015deep, song2019generative, ho2020denoising, song2020denoising}, the current state-of-the-art generative models, have delivered outstanding results in a myriad of tasks \cite{chen2020wavegrad, nichol2021glide, rombach2022high, saharia2022photorealistic}, primarily due to the scalable and stable training methodologies \cite{dhariwal2021diffusion}. In a significant leap forward, \citet{song2020score} introduced a general framework that encapsulates the essence of previous diffusion models through the stochastic differential equations (SDEs), which, equivalently, correspond to the neural ordinary differential equations (ODEs) \cite{chen2018neural} in the sense of probability flow. Recently, \citet{lipman2022flow}, developed the Flow matching (FM), a scalable, simulation-free approach to train the probability flow, also known as continuous normalizing flow \cite{chen2018neural, grathwohl2018ffjord}, by directly regressing vector fields along specific conditional probability paths. We note that two concurrent studies, the stochastic interpolant by \citet{albergo2022building} and the rectified flow by \citet{liu2022flow}, propose similar methodologies for matching distributions using flows, albeit from distinct viewpoints.
    
    However, during inference, generating a high-quality sample via simulating the learned ODEs often requires multiple function evaluations, leading to a long inference time. This inefficiency arises from the utilization of independent couplings that overlook the intrinsic structures connecting source and target distributions \cite{lipman2022flow, liu2022flow}. To mitigate this issue, there has been a shift towards designing non-trivial couplings inspired by optimal transport theory \cite{pooladian2023multisample, tong2023conditional, tong2023improving} or learning a coupling based on an auxiliary VAE-style objective function to minimize the trajectory curvature \cite{lee2023minimizing}. Notably, these existing continuous-time generative models, have predominantly adopted a uniform/single ODE to model the transportation process between two (empirical) distributions.
    
    \textbf{Contributions}. We introduce Switched FM (SFM), a generalized framework that eliminates singularities via switching ODEs as opposed to employing a single ODE in FM. The core principle of SFM is that according to the inherent (joint) heterogeneity of the underlying distributions, i.e., (jointly) dependent on the source or/and target data samples,  a specific ODE should be selected from the pool of the candidate ODEs to facilitate the transportation process while preserving the marginal vector fields or probability paths.
    
    To summarize, the major contributions of this study are multi-folded, including:
    \begin{enumerate}%[leftmargin=*, topsep=0mm, itemsep=0mm, parsep=0mm]
      \item \textbf{Development of SFM}: We establish SFM, a versatile continuous-time generative model that eliminates singularities encountered in the FM via switching the candidate ODEs, and allows the intersection of probability paths from different ODEs.
      \item \textbf{Theoretical insights}: Through rigorous analysis, we demonstrate that FM struggles with transporting between simple distributions due to the existence and uniqueness of initial value problems of ODEs while such limitation can be effectively addressed by SFM, offering a more efficient solution.
      \item \textbf{Integration with advanced techniques}: SFM can seamlessly integrate with the existing advanced techniques, for example, minibatch optimal transport, to further enhance the straightness of the flow, facilitating a more efficient sampling process.
      \item \textbf{Empirical validation}: We validate the effectiveness of the newly proposed SFM through extensive experiments on both synthetic and real-world datasets, achieving competitive or even better performance compared to existing methods, such as FM.
    \end{enumerate}
    
    \textbf{Organization}. The rest of this article is organized as follows. Section~\ref{sec2} introduces some preliminaries on (neural) ODEs, continuous normalizing flows, flow matching, and optimal transport. In Sec.~\ref{sec:limitations}, we theoretically show the limitations of FM. Then,  we present the SFM in Sec.~\ref{sec:SFM}. Related works are discussed in Sec.~\ref{sec:related_works}. In Sec.~\ref{sec:experiments},  we provide numerical verifications on synthetic and real-world datasets. Finally, we conclude the article in Sec.~\ref{sec:conclusion}, and all the details of this work are found in the appendices.
    
    \textbf{Notations.} Before ending this section, we provide the following notations that will be used throughout the article: $\mathbb{R}$ (resp. $\mathbb{R}^+$) -- the set of (resp. positive) real numbers; $\mathbb{R}^d$ -- the Euclidean space; $\|\cdot\|$ -- the $d$-dimensional ($d$-d) Euclidean norm; $\nabla$ and $\nabla\cdot$ -- the gradient and divergence operator, respectively; $\bm{1}_d$ -- the $d$-d vector with all elements being $1$; $\bm{I}_d$ -- the $d$-d identity matrix;  $\text{Tr}(A)$ -- the trace of the square matrix $A\in \mathbb{R}^{d\times d}$;
    $\delta_{\bm{x}}$ -- the Dirac mass at the point $\bm{x}\in \mathbb{R}^d$; $\mathcal{P}(\mathbb{R}^d)$ -- the space of Borel probability measures on $\mathbb{R}^d$; For given $q_0 \in \mathcal{P}(\mathbb{R}^d)$ and $q_1 \in \mathcal{P}(\mathbb{R}^d)$, then $\Pi(q_0, q_1)$ is defined as the set of all joint probability measures on $\mathbb{R}^d\times \mathbb{R}^d$ whose marginals are $q_0$ and $q_1$, and $q \in \Pi(q_0, q_1)$ is called a coupling between $q_0$ and $q_1$; $\mathcal{U}(a, b)$ -- the uniform distribution over the interval $[a, b]$; $\mathcal{N}(\bm{\mu}, \bm{\Sigma})$ -- the multivariate Gaussian distribution with the mean vector $\bm{\mu}$ and the covariance matrix $\bm{\Sigma}$; $\mathcal{H}^d$ -- the $d$-d Hausdorff measure (with suitable normalization); $|S|$ -- the cardinality of the set $S$.
    
    \section{Preliminaries}
    \label{sec2}
    
    \subsection{ODE and Probability Flows}
    \begin{definition}[\citet{o2006metric,villani2009optimal}]
      A map $\bm{f}: \mathcal{X} \rightarrow \mathcal{Y}$ between metric spaces $(\mathcal{X}, d_{\mathcal{X}})$ and $(\mathcal{Y}, d_{\mathcal{Y}})$ is said to be Lipschitz continuous (or $L$-Lipschitz) if $d_{\mathcal{Y}}[\bm{f}(\bm{x}), \bm{f}(\bm{x}^{\prime})] \leq L d_{\mathcal{X}}(\bm{x}, \bm{x}^{\prime})$ for all $\bm{x}$, $\bm{x}^{\prime}$ in $\mathcal{X}$. The best admissible constant $L$ is called the Lipschitz constant of $\bm{f}$, denoted by $\|\bm{f}\|_{\text{Lip}}$.
    \end{definition}
    The Cauchy problem or the initial value problem (IVP) is defined as the time-dependent Ordinary Differential Equation (ODE) of the following general form:
    \begin{equation}
      \label{eq_ODE}
      \frac{{\rm d} \bm{x}(t)}{{\rm d} t} = \bm{u}_t(\bm{x}), \quad t\in[0, 1], \quad \bm{x}(0) =\bm{x}_0,
    \end{equation}
    where $\bm{u}_t(\bm{x}): [0, 1]\times \mathbb{R}^d \rightarrow \mathbb{R}^d$ is a smooth\footnote{In this work, we assume that the vector field is (locally) Lipschitz continuous in both arguments $t$ and $\bm{x}$ and thereby the Picard's existence theorem \cite{arnold1992ordinary} guarantees the existence and uniqueness of the solution locally defined on a maximal time interval.} vector field. The solution $\bm{x}(t)$ of this ODE~\eqref{eq_ODE} induces a map, called the time-dependent flow: $\bm{\phi}_t(\bm{x}_0): [0, 1]\times \mathbb{R}^d \rightarrow \mathbb{R}^d$, defined as $\bm{\phi}_t(\bm{x}_0):=\bm{x}(t)$. For a given initial distribution $\bm{x}_0\sim q_0(\bm{x}_0)$, the above ODE~\eqref{eq_ODE} induces the associated probability flows $p_t(\bm{x}): [0, 1]\times \mathbb{R}^d \rightarrow \mathbb{R}^+$, satisfying the continuity equation \cite{pedlosky2013geophysical}:
    \begin{equation}
      \label{eq_continuity}
      \frac{\partial  p_t(\bm{x})}{\partial t} = -\nabla \cdot [p_t(\bm{x})\bm{u}_t(\bm{x})],
    \end{equation}
    with the initial condition  $p_0(\bm{x}_0) = q_0(\bm{x}_0)$.  Typically, $(\bm{\phi}_t)_{\#}p_0$ stands for the image measure or push-forward of $p_0$ by $\bm{\phi}_t$. In addition, if, for a given target distribution $\bm{x}_1\sim q_1(\bm{x}_1)$, it holds $p_1(\bm{x}_1) = q_1(\bm{x}_1)$, then the set of all these vector fields satisfying the boundary conditions is defined as $U(q_0, q_1)$.
    
    \subsection{Continuous Normalizing Flow}
    \citet{chen2018neural} proposed a continuous-time generative model, called the Continuous Normalizing Flow (CNF), that can be trained via performing maximum likelihood estimation. Specifically, the generative process works by first sampling data points from the source distribution  $\bm{x}_0 \sim q_{0}(\bm{x}_0)$. Then, these data points are transformed into different ones by solving the initial value problem of the neural ODE (NODE) \cite{chen2018neural}:
    \begin{equation}
      \label{eq_NODE}
      \frac{{\rm d} \bm{x}(t)}{{\rm d} t} = \bm{v}_t(\bm{x};\bm{\theta}), \quad t\in[0, 1], \quad \bm{x}(0) =\bm{x}_0, \\
    \end{equation}
    where $\bm{v}_t(\bm{x};\bm{\theta})$ is a parameterized neural network with the trainable weights $\bm{\theta}$ and the flow map is defined as $\bm{\varphi}_t(x_0;\bm{\theta})$. The object is that the final states $\bm{x}(1)$ from the above ODE~\eqref{eq_NODE} should constitute the target data instances. In addition, based on the instantaneous change of variables formula \cite{chen2018neural}, the change in log probability follows a second ODE:
    \begin{equation}
      \label{eq2}
      \frac{{\rm d}{\log p_t(\bm{x})}}{{\rm d} t} = -\text{Tr}\left[\frac{\partial  \bm{v}_t(\bm{x};\bm{\theta})}{\partial \bm{x}}\right],
    \end{equation}
    resulting in the total change in log density as follows:
    \begin{equation}
      \label{eq3}
      \log p_{1}(\bm{x}) = \log q_0(\bm{x}_0) -\int_{0}^{1}\text{Tr}\left[\frac{\partial  \bm{v}_t(\bm{x};\bm{\theta})}{\partial \bm{x}}\right] {\rm d}t.
    \end{equation}
    Finally, the CNF can be trained by maximizing \eqref{eq3}. We note that the CNF requires simulating the ODEs~\eqref{eq_NODE} and \eqref{eq2} during training, yielding high computational costs.
    
    \subsection{(Conditional) Flow Matching}
    Different from the training of the CNF as well as its objective, \citet{lipman2022flow} proposed Flow Matching (FM), a simple simulation-free training method that employs a stable objective by regressing a target vector field $\bm{u}_t(\bm{x})$ that generates the desired probability paths $p_t(\bm{x})$, satisfying $p_0[\bm{x}(0)]= q_0(\bm{x}_0)$ and $p_1[\bm{x}(1)]= q_1(\bm{x}_1)$. Then, the regression objective is
    \begin{equation}
      \label{obj_FM}
      \mathcal{L}_{\text{FM}}(\bm{\theta})=\mathbb{E}_{t, p_t(\bm{x})} \left\|\bm{v}_t(\bm{x};\bm{\theta}) - \bm{u}_t(\bm{x})\right\|^2,
    \end{equation}
    where $t\sim \mathcal{U}(0, 1)$ and $\bm{x}(t) \sim p_t(\bm{x})$.
    Ideally, when the above objective \eqref{obj_FM} approaches zero, the learned vector field $\bm{v}_t(\bm{x};\bm{\theta})$ will generate $p_t(\bm{x})$. However, this objective \eqref{obj_FM} is, in general, computationally intractable without knowing the explicit forms of $\bm{u}_t(\bm{x})$ and $p_t(\bm{x})$.
    
    Regarding this intractable issue, Conditional FM (CFM) \cite{lipman2022flow, pooladian2023multisample, tong2023conditional, tong2023improving} employs a simpler and tractable regression objective to effectively learn the vector field $\bm{v}_t(\bm{x};\bm{\theta})$ by incorporating a latent condition $\bm{z}$:
    \begin{equation}
      \label{obj_CFM}
      \mathcal{L}_{\text{CFM}}(\bm{\theta})=\mathbb{E}_{t, q(\bm{z}), p_t(\bm{x}|\bm{z})}
      \| \bm{v}_t(\bm{x};\bm{\theta})
      - \bm{u}_t(\bm{x}|\bm{z})\|^2,
    \end{equation}
    which has the same gradient, w.r.t. $\bm{\theta}$ as the FM objective~\eqref{obj_FM} \cite{lipman2022flow, pooladian2023multisample, tong2023conditional, tong2023improving}.
    Usually, $q(\bm{z})$ is chosen as an independent coupling between two distributions, i.e.,
    \begin{equation}
      \label{eq:ind_coupling}
      q(\bm{z}):=q(\bm{x}_0, \bm{x}_1)=q_0(\bm{x}_0)q_1(\bm{x}_1),
    \end{equation}
    with $\bm{x}(t)$ being the linear interpolation of $\bm{x}_0$ and $\bm{x}_1$:
    \begin{equation}
      \bm{x}(t)= (1-t)\bm{x}_0 + t \bm{x}_1,
    \end{equation}
    resulting in a constant speed vector field given $\bm{z}$:
    \begin{equation}
      \label{eq:ind_VF}
      \bm{u}_t(\bm{x}|\bm{z})=\bm{x}_1 - \bm{x}_0.
    \end{equation}
    This specific CFM model was also extensively investigated in prior research, notably in studies such as \citet{liu2022flow,albergo2022building}, where it is referred to as the rectified flow or the stochastic interpolant. In addition, $q(\bm{z})$ can be also selected as the (minibatch) optimal transport coupling \cite{fatras2019learning, pooladian2023multisample, tong2023conditional, tong2023improving}.  Here, we call these two methods independent CFM (I-CFM) and optimal transport CFM (OT-CFM).
    
    \subsection{Static and Dynamic Optimal Transport}
    The (static) optimal transport theory
    \cite{villani2009optimal, santambrogio2015optimal, peyre2019computational}, a field in mathematics, focuses on efficiently transferring one distribution to another. Usually, the optimal transport cost between two measures is defined as the Kantorovich problem \cite{kantorovich1942translocation}, which can be described as follows:
    \begin{equation}\label{eq_KP}
      C(q_0, q_1) =\inf\limits_{\pi \in \Pi(q_0, q_1)} \int c(\bm{x}_0, \bm{x}_1) {\rm d} \pi(\bm{x}_0, \bm{x}_1),
    \end{equation}
    where $c(\bm{x}_0, \bm{x}_1)$ is the cost for transporting one unit of mass from $\bm{x}_0$ to $\bm{x}_1$. In this paper, we consider the cost defined in terms of Euclidean distance, resulting in the following squared $2$-Wasserstein distance:
    \begin{equation}
      \label{eq:OT}
      W(q_0, q_1)^2 = \inf\limits_{\pi \in \Pi(q_0, q_1)} \int \|\bm{x}_0 - \bm{x}_1\|^2 {\rm d} \pi(\bm{x}_0, \bm{x}_1).     
    \end{equation}
      Notably, the squared $2$-Wasserstein distance has the equivalent dynamic form, known as the Benamou-Brenier formula \cite{benamou1999numerical, brenier2003extended, villani2009optimal}:
      \begin{equation}
        \label{eq:dyn_OT}
        W(q_0, q_1)^2 = \inf\limits_{\bm{u}_t \in U(q_0, q_1)}  \int_0^1 \int p_t(\bm{x})\|\bm{u}_t(\bm{x})\|^2 {\rm d} \bm{x} {\rm d} t.
      \end{equation}
      
      \begin{figure}[t]
        % \vskip -0.2in
        \begin{center}
          \centering \subfigure{\label{fig_uniform_a}}\subfigure{\label{fig_uniform_b}}
          \includegraphics[width=0.49\textwidth]{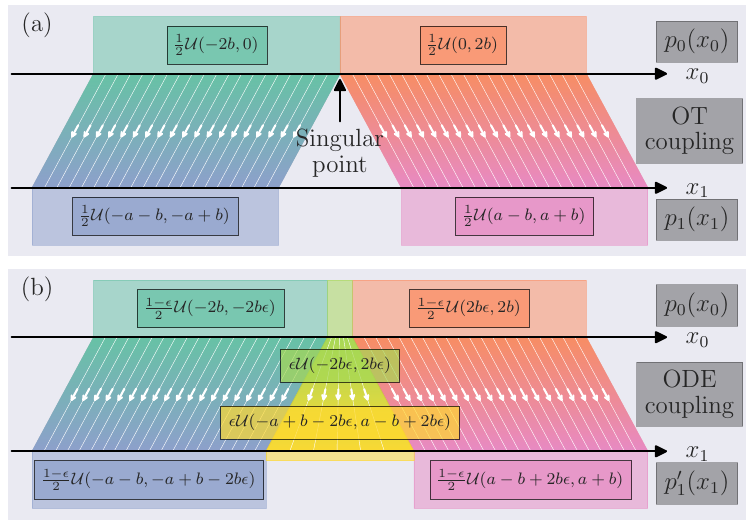}
          % \vskip -0.1in
          \caption{Illustration of the optimal transport (OT) coupling (a) and the ODE coupling (b) on the example in Proposition~\ref{thm:uniform21}.}
          \label{fig_uniform}
        \end{center}
        % \vskip -0.2in
      \end{figure}

      \begin{figure}[t]
        % \vskip -0.2in
        \begin{center}
          \centering \subfigure{\label{fig_uniform22_a}}\subfigure{\label{fig_uniform22_b}}
          \includegraphics[width=0.49\textwidth]{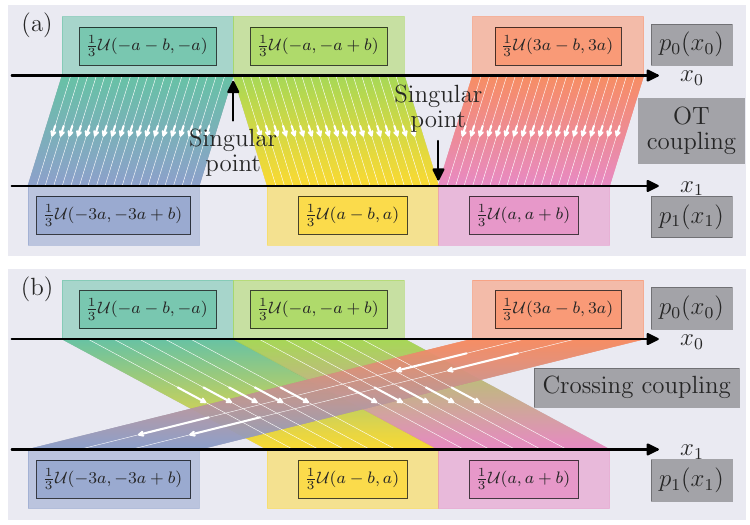}
          % \vskip -0.1in
          \caption{Illustration of the OT coupling (a) and the crossing coupling (b) on the example in Proposition~\ref{thm:uniform22}.}
          \label{fig_uniform22}
        \end{center}
        % \vskip -0.2in
      \end{figure}
      
      \section{Limitations of Folw Matching}
      \label{sec:limitations}
      In reality, the inherent (joint) heterogeneity of the source or/and target distributions may lead to a scenario where even an optimally trained FM model exhibits pronounced singularity. Consequently, this section aims to theoretically elucidate the limitations inherent to FM models through a series of propositions. All the details of the proofs are relegated to the appendices.
      \begin{propositionE}[Heterogeneity in $q_0$ or $q_1$][end,restate]\label{thm:uniform21}
        Suppose the source distribution $q_0$ is an $1$-d uniform distribution $q_0=\mathcal{U}(-2b, 2b)$ and the target distribution $q_1$ is an $1$-d uniform mixture ($2$-modes) $q_1=\frac{1}{2}\mathcal{U}(-a-b, -a+b) + \frac{1}{2}\mathcal{U}(a-b, a+b)$, where $a \gg b \geq 0$. Consider the (dynamic) optimal transport problem as defined in Eq.~\eqref{eq:OT} (or Eq.~\eqref{eq:dyn_OT}).
        \begin{enumerate}%[leftmargin=*, topsep=0mm, itemsep=0mm, parsep=0mm]
          \item \label{thm1_it1} If the NODE \eqref{eq_NODE} exactly\footnote{To be precise, $q_0$ can be completely transported to $q_1$ with the minimum of the squared $2$-Wasserstein distance \eqref{eq:dyn_OT}.} solves the problem,  then $x(0)=0$ is a singular point, i.e., where the flow map $\varphi_1(0;\bm{\theta}): x(0)=0\rightarrow x(1)$ is not well-defined or discontinuous (with two directions to $q_1$), as shown in Fig.~\ref{fig_uniform_a}.
          \item \label{thm1_it2} If the NODE \eqref{eq_NODE} approximately\footnote{A small fraction ($\epsilon \ll 1$) of the mass cannot be transferred from the source $q_0$ to the target $q_1$.}
          solves the problem, resulting in an approximated target distribution $q_1^{\prime}$,  then there is a neighborhood $O$ of $x(0)=x_0$ which is homeomorphically mapped to the open subset in target space connecting the two modes, as shown in Fig.~\ref{fig_uniform_b}.
          \item \label{thm1_it3} If the two modes of $q_1$ are far away from each other, i.e., $a\gg 1$, then the flow map $\varphi_1[x_0;\bm{\theta}]$ within a neighborhood $O$ as defined in the above-approximated NODE (the second bulletin) has a large Lipchitz constant.
        \end{enumerate}
      \end{propositionE}
      \begin{proofE}
        We routinely prove these three bulletins in the following.
        \begin{enumerate}
          \item If the NODE \eqref{eq_NODE} exactly solves the optimal transport problem, then the flow map $\varphi_1(x_0;\bm{\theta})$ transports all mass from the source to the target and simultaneously achieves the minimum squared $2$-Wasserstein distance. By construction of the source and target distributions, the support of the source distribution should be divided into two equal parts at the point $0$, and each part is transported to one of the two disconnected components of the target support. In addition, from Lemma~\ref{appd:lem_mon} and Theorem~\ref{appd:thm_mon}, the left (resp., right) part of the source support should be transported to the left part of the target support and this optimal transport map should satisfy the monotonic property \eqref{appd:eq_mon}, as illustrated in Fig.~\ref{fig_uniform_a}. Hence, the flow map $\varphi_1(0;\bm{\theta}): x(0)=0\rightarrow x(1)$ is not well-defined or discontinuous at the point $0$.
          \item  If the NODE \eqref{eq_NODE} approximately solves the problem, resulting in an approximated target distribution $q_1^{\prime}$, then there exists a small fraction of the mass that cannot be transferred from the source to the target. More precisely, based on the homeomorphism of the flow map $\varphi_1(x_0;\bm{\theta})$ (Theorem~\ref{appd:thm_hom}), it should map the source connected support to another connected set, thereby preserving connectedness. However, the support of the target distribution has two disconnected components, implying that there is a neighborhood $O$ of $x(0)=x_0$  in the source support, which is homeomorphically mapped to the open subset connecting these two modes.
          \item If the two modes of $q_1$ are far away from each other, i.e., $a\gg 1$, then it implies that the flow map $\varphi_1(x_0;\bm{\theta})$ maps the neighborhood $O$ of $x(0)=x_0$ to a large open subset connecting these two largely shifted modes, though the transported mass can be still small (say $\epsilon \ll 1$). Hence, the flow map $\varphi_1(x_0;\bm{\theta})$ has a large Lipchitz constant $L^{\prime}$, satisfying
          \begin{equation}
            \|\varphi_1(\bar{x}_0;\bm{\theta}) - \varphi_1(\hat{x}_0;\bm{\theta}) \| \leq L^{\prime} \|\bar{x}_0 - \hat{x}_0\|.
          \end{equation}
          More precisely, the order of $L^{\prime}$ is $\mathcal{O}(\frac{a}{\epsilon})$. In addition, based on the Gronwall's inequality~\eqref{appd:eq_gronwall}, the Lipchitz constant $L$ of the vector field $\bm{v}_t(\bm{x};\bm{\theta})$ should satisfy that
          \begin{equation}
            \|\varphi_1(\bar{x}_0;\bm{\theta}) - \varphi_1(\hat{x}_0;\bm{\theta}) \| \leq e^L \|\bar{x}_0 - \hat{x}_0\|,  \quad \forall~ \bar{x}_0, \hat{x}_0 \in O,
          \end{equation}
          implying that the order of $L$ is $\mathcal{O}(\ln{L^{\prime}})=\mathcal{O}\left(\ln{\frac{a}{\epsilon}}\right)$. Therefore, when $a\gg 1$, the flow map $\varphi_1[x_0;\bm{\theta}]$ within a neighborhood $O$ has a large Lipchitz constant $L$ of the order $\mathcal{O}\left(\ln{\frac{a}{\epsilon}}\right)$.
        \end{enumerate}
        The proof is complete.
      \end{proofE}
      \begin{remark}
        In Proposition~\ref{thm:uniform21}, without loss of generality, we only consider the target distribution $q_1$ with heterogeneity (two modes). The intuition behind the theoretical results is simple. Within the context of bulletin~\ref{thm1_it1} from Proposition~\ref{thm:uniform21}, the mechanism of optimal transport coupling necessitates the division of $p_0$ into two symmetrical segments at the juncture $x(0)=0$, directing these segments towards the dual modes of $q_1$. This process engenders a singularity at $x(0)=0$, a direct consequence of $q_1$'s heterogeneity.  For the bulletins~\ref{thm1_it2}~\&~\ref{thm1_it3} of Proposition~\ref{thm:uniform21}, these statements are the consequences of the flow map $\varphi_t(x(0);\bm{\theta})$ being a homomorphism (more precisely, diffeomorphism), i.e. a bijective and continuous function whose inverse is also continuous.
      \end{remark}
      A straightforward corollary emerging from Proposition~\ref{thm:uniform21} is articulated as follows.
      \begin{corollaryE}[][end,restate]
        \label{thm:corol}
        Given the discrete distributions $q_0=\delta_0$ and $q_1=\frac{1}{2}\delta_{-a} + \frac{1}{2}\delta_a$,  consider the optimal coupling $q(x_0=0, x_1=\pm a)=\frac{1}{2}$, then it cannot be solved by an ODE. Furthermore, the learned flow map $\varphi_1(0;\bm{\theta})$ transfers the initial Dirac mass to some point $a^{\prime}$ in the open set $(-a, a)$, i.e., $q_1^{\prime} =\delta_{a^{\prime}}$.
      \end{corollaryE}
      \begin{proofE}
        Based on the optimal coupling $q(x_0=0, x_1=\pm a)=\frac{1}{2}$, the mass at the point $0$ should be divided into two halves, and one half is transported to $-a$, and the other half is transported to $a$. However, based on the existence and uniqueness of ODEs' solutions, the flow map induced by the ODE is a determined map that can only transport the mass at the point $0$ to some point $a^{\prime}:=\varphi_1(0;\bm{\theta})$.
        \begin{figure}[htb]
          % \vskip 0.2in
          \begin{center}
            \centering
            \includegraphics[width=0.6\textwidth]{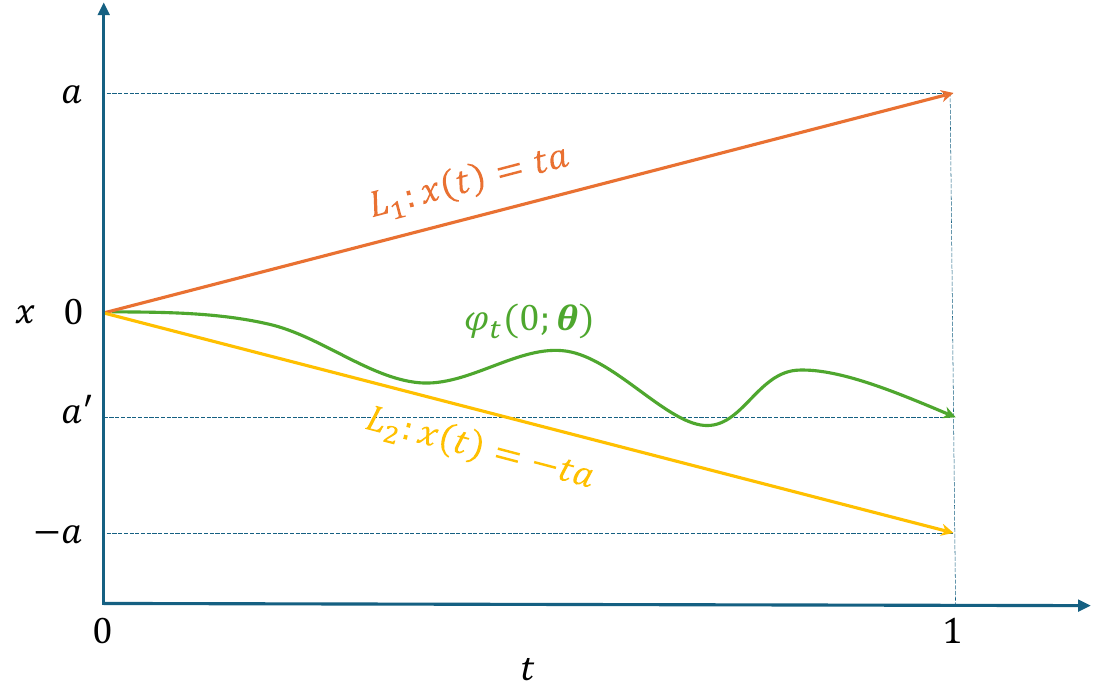}
            \caption{Illustration of the sandwiched property of ODEs in Corollary~\ref{thm:corol}.}
            \label{fig_corr}
          \end{center}
        \end{figure}
        Now,  we claim $a^{\prime} \in (-a, a)$. The optimal coupling $q(x_0=0, x_1=\pm a)=\frac{1}{2}$ shows the singularity at the point $x(0)=0$ with two directions. Hence, in an average way by eliminating the downward and upward components of these two directions, respectively, it leads to a horizon direction (between them) at the point $x(0)=0$. However, for any point belonging to the lines $x(t) = ta$ or $x(t) = -ta$, $t\in(0, 1]$, the direction of this point is determined along the corresponding line. Moreover, by the Theorem~\ref{appd:thm_noncross}, the ideal learned solution $\varphi_t(0;\bm{\theta})$ is always sandwiched between these two lines, as illustrated in Fig.~\ref{fig_corr}. Therefore, the learned flow map $\varphi_1(0;\bm{\theta})$ transfers the initial Dirac mass to some point $a^{\prime}$ in the open set $(-a, a)$, i.e., $q_1^{\prime} =\delta_{a^{\prime}}$.
      \end{proofE}
      \begin{remark}
        Intuitively, to resolve the issue identified in Corollary~\ref{thm:corol}, the flow map should assign the initial state to two disparate target states, thereby challenging the existence and uniqueness theorem of the IVP for a smooth ODE.
      \end{remark}

      \begin{figure*}[t]
        \begin{center}
          \centering
          \subfigure{\label{fig_Infinite_a}}
          \subfigure{\label{fig_Infinite_b}}
          \subfigure{\label{fig_Infinite_c}}
          \includegraphics[width=0.98\textwidth]{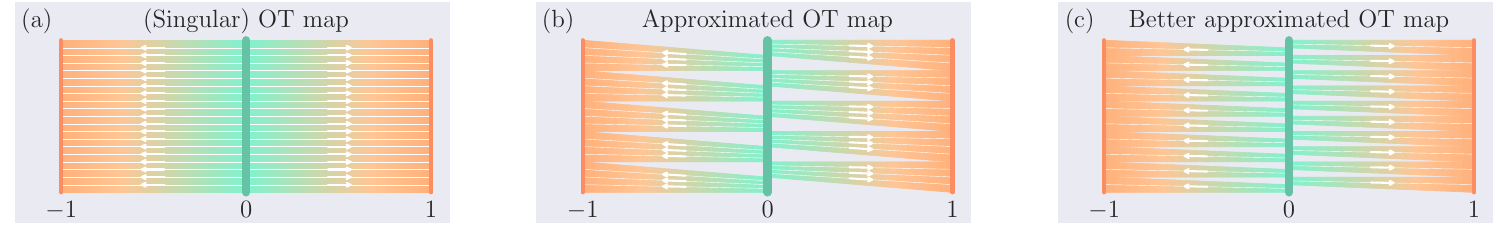}
          % \vskip -0.1in
          \caption{Illustration of the (singular) OT map (a) and the (better) approximated OT maps (b) \& (c) on the example in Proposition~\ref{thm:infinite}.}
          \label{fig_Infinite}
        \end{center}
        % \vskip -0.2in
      \end{figure*}
      
      \begin{propositionE}[Heterogeneity in both $q_0$ and $q_1$][end,restate]\label{thm:uniform22}
        Suppose the source and target distributions $q_0$ and $q_1$ are two different $1$-d uniform mixtures ($2$-modes), respectively, i.e., $q_0=\frac{2}{3}\mathcal{U}(-a-b, -a +b) + \frac{1}{3}\mathcal{U}(3a-b, 3a)$ and $q_1=\frac{1}{3}\mathcal{U}(-3a, -3a + b) + \frac{2}{3}\mathcal{U}(a-b, a+b)$, where $a \gg b \geq 0$. Consider the (dynamic) optimal transport problem as defined in Eq.~\eqref{eq:OT} (or Eq.~\eqref{eq:dyn_OT}). If the NODE~\eqref{eq_NODE} exactly solves the problem, then $x(0)=-a$ (reps., $x(1)=a$) is a singular point as shown in Fig.~\ref{fig_uniform22_a}.
      \end{propositionE}
      \begin{proofE}
        The proof is similar to the one of the Proposition~\ref{thm:uniform21}. Different from the case  Proposition~\ref{thm:uniform21}, here, the disconnectedness occurs in both the source distribution $q_0$ and the target distribution$q_1$, where each disconnected component of their supports has a distinct mass.  The optimal transport coupling should split both the large components in the supports of the source and the target distributions into two equal parts, resulting in a total of three equal components, respectively. Therefore, the ODE is thereby singular at the split points $x(0)=-a$ and $x(1)=a$.
      \end{proofE}
      \begin{remark}
        The conceptual underpinnings of Proposition~\ref{thm:uniform22} closely mirror that of Proposition~\ref{thm:uniform21}. However, the identified singularity originates from the heterogeneity present in both $q_0$ and $q_1$ under the optimal coupling induced by the squared 2-Wasserstein distance~\eqref{eq:OT}. Instead, a crossing coupling, illustrated in Fig.~\ref{fig_uniform22_b}, enables an exact transportation between two large (resp., small) modes of $q_0$ and $q_1$, adeptly sidestepping any potential singularities. This coupling is locally optimal given the source and target modes, although it does not constitute a global optimum. Regrettably, achieving such coupling via a single ODE is impossible, as ODE trajectories cannot intersect \cite{arnold1992ordinary,dupont2019augmented,zhang2020approximation,massaroli2020dissecting,zhu2021neural,liu2022flow}.
      \end{remark}
      
      \begin{propositionE}[Infinite number of singular points][end,restate]\label{thm:infinite}
        Suppose the source and target distributions $q_0$ and $q_1$ are defined on $\mathbb{R}^2$ with $q_0$ being $\mathcal{H}^1$ restricted to $\{0\} \times [-1, 1]$, and $q_1$ being $(1/2)\mathcal{H}^1$ restricted to $\{-1, 1\} \times [-1, 1]$, respectively. Consider the (dynamic) optimal transport problem as defined in Eq.~\eqref{eq:OT} (or Eq.~\eqref{eq:dyn_OT}). If the NODE~\eqref{eq_NODE} exactly solves the problem, then all the points $\bm{x}(0)=(0, a), a \in [-1, 1]$ are singular points as shown in Fig.~\ref{fig_Infinite_a}.
      \end{propositionE}
      \begin{proofE}
        Indeed, to exactly solve the optimal transport problem, each source point $(0, a)$ has two expected target points $(-1, a)$ and $(1, a)$, since both of them require the same unit cost for $(0, a)$ to one of them. Therefore, the optimal flow map $\bm{\varphi}_1[(0, a);\bm{\theta}]: \bm{x}(t=0)=(0, a)\rightarrow \bm{x}(1)$ is not well-defined (with two directions to $q_1$).
      \end{proofE}
      \begin{remark}
        We note that Proposition~\ref{thm:infinite} presents a quintessential example often employed to demonstrate the existence of a Monge minimizer, as detailed in\cite{villani2009optimal}. To achieve an optimal cost, one must split the mass at $(0, a)$ into two equal parts, and subsequently advance one towards $(-1, a)$ and the other towards $(1, a)$. Although this procedure does not yield a conventional map (or Monge transport), one can approximate it via a discontinuous map with finite singular points as shown in Fig.~\ref{fig_Infinite_b}. In addition, it is always possible to construct a better map (see Fig.~\ref{fig_Infinite_c}) by similarly incorporating additional singular points.
      \end{remark}
      \begin{remark}
        It is worth noting that the dimensionality of the manifold corresponding to the singular points in Proposition~\ref{thm:infinite} is $1$. Conversely, Propositions~\ref{thm:uniform21} and \ref{thm:uniform22} are characterized by a manifold dimensionality of $0$. In higher-dimensional cases, the singular points are encompassed within a stratified union of manifolds with distinct dimensions \cite{caffarelli1977regularity, caffarelli1998obstacle, figalli2019fine}. To eliminate these singularities, it is essential to ensure that the cost functions, the spaces, and the probability measures meet adequate regularity assumptions, but this is often not the case when dealing with real-world data.
      \end{remark}
      
      \section{Switched Flow Matching}
      \label{sec:SFM}
      Inspired by the limitations of FM, we construct a new class of continuous-time generative models, referred as to Switched FM (SFM) which solves the transport problem between source and target distributions via switching multiple ODEs, particularly eliminating the singularities encountered in FM using a single ODE.  The comparison of the FM and SFM are summarized in Table \ref{tab:prob_path_summary}.
      
      \newcommand{\no}[0]{\color{red}\ding{55}}
      \newcommand{\yes}[0]{\color{green}\ding{51}}
      
      \begin{table}[htb]
        \centering
        \caption{Properties for the ODE-based generative models, including the FM, CFM, and our proposed SFM. Particularly, the SFM can not only handle general source distributions, and optimal transport flows (OT-SFM), but also employ multiple ODEs to eliminate the singularity, allowing the intersection of trajectories from different ODEs, and owning the relatively good regularity.}
        \vspace*{1em}
        \label{tab:prob_path_summary}
        \resizebox{\linewidth}{!}{        \begin{tabular}{@{}l|ccccc}
        \toprule
        \toprule
        ODE model & General source  & OT & Mult. ODEs & Intersection & Regularity\\
        \midrule
        FM      &  \no & \no & \no & \no & \no \\
        I-CFM &  \yes & \no & \no & \no & \no \\
        OT-CFM & \yes & \yes & \no & \no & \no \\
        \midrule
        I-SFM  & \yes & \no & \yes & \yes & \yes \\
        OT-SFM  & \yes & \yes & \yes & \yes & \yes\\
        \bottomrule
        \bottomrule
      \end{tabular}
      }
    \end{table}
    
    \subsection{Formulation}
    Consider the source (resp., target) distribution, denoted as $q_0(\bm{x})$ (resp., $q_1(\bm{x})$), which is modeled as a mixture of conditional distributions $q_0(\bm{x}|\bm{s})$  (resp., $q_1(\bm{x}|\bm{s})$) that vary in response to a latent conditioning variable $\bm{s}$, termed the switching signal. Mathematically, this is expressed as:
    \begin{equation}
      \label{eq:marginal_PP_boundary}
      q_i(\bm{x}) = \int q_i(\bm{x}|\bm{s}) q^{\circ}(\bm{s})  {\rm d} \bm{s}, \quad i\in\{0, 1\},
    \end{equation}
    where $q^{\circ}(\bm{s})$ represents the distribution over the switching signal. Correspondingly, the marginal probability path $p_t(\bm{x})$ is modeled as a mixture of probability paths $p_t(\bm{x}|\bm{s})$ of the following form:
    \begin{equation}
      \label{eq:marginal_PP}
      p_t(\bm{x}) = \int p_t(\bm{x}|\bm{s}) q^{\circ}(\bm{s})  {\rm d} \bm{s},
    \end{equation}
    where $p_t(\bm{x}|\bm{s})$ should satisfy the boundary conditions, i.e., $p_0(\bm{x}|\bm{s})=q_0(\bm{x}|\bm{s})$ and $p_1(\bm{x}|\bm{s})=q_1(\bm{x}|\bm{s})$, implying $p_0(\bm{x})=q_0(\bm{x})$ and $p_1(\bm{x})=q_1(\bm{x})$. We assume that each conditional probability path $p_t(\bm{x}|\bm{s})$ arises from a corresponding conditional vector field $\bm{u}_t(\bm{x}|\bm{s})$. Significantly, our proposed SFM involves switching these ODEs rather than relying on a single ODE in FM~\eqref{obj_FM}. The corresponding sampling process is formalized as follows.
    \begin{propositionE}[Switching ODEs][end,restate]\label{thm:Switching_ODEs}
      The marginal probability path $p_t(\bm{x})$ can be effectively sampled by switching ODEs in the following three steps:
      \begin{enumerate}%[leftmargin=*, topsep=0mm, itemsep=0mm, parsep=0mm]
        \item \textbf{Sampling an ODE}. Sampling a switching signal $\bm{s}$ from the distribution $q^{\circ}(\bm{s})$, resulting in the specified ODE $\bm{u}_t(\bm{x}|\bm{s})$;
        \item \textbf{Sampling an initial state}. Sampling an initial state $\bm{x}_0$ (resp., backward one $\bm{x}_1$) from the conditional distribution $q_0(\bm{x}_0|\bm{s})$ (resp., $q_1(\bm{x}_1|\bm{s})$);
        \item \textbf{Solving the IVP}. Generating the corresponding conditional probability path $p_t(\bm{x}|\bm{s})$ by the vector field $\bm{u}_t(\bm{x}|\bm{s} )$ from the initial state $\bm{x}_0$ (resp., $\bm{x}_1$).
      \end{enumerate}
    \end{propositionE}
    \begin{proofE}
      Here, we aim to prove that the marginal probability path $p_t(\bm{x})$ can be equivalently sampled by switching ODEs in the above three steps.
      
      By construction,  we introduce a latent conditioning variable $\bm{s}$ to represent the source (resp., target) distribution $q_0(\bm{x})$ (resp., $q_1(\bm{x})$) as a mixture of conditional distributions $q_0(\bm{x}|\bm{s})$  (resp., $q_1(\bm{x}|\bm{s})$),  satisfying (also see Eq.~\eqref{eq:marginal_PP_boundary})
      \begin{equation}\label{appd:eq_marginal_PP_boundary}
        q_i(\bm{x}) = \int q_i(\bm{x}|\bm{s}) q^{\circ}(\bm{s})  {\rm d} \bm{s}, \quad i\in\{0, 1\},
      \end{equation}
      where $q^{\circ}(\bm{s})$ is the distribution over the switching signal. In addition, we model the marginal probability path $p_t(\bm{x})$ as a mixture of probability paths $p_t(\bm{x}|\bm{s})$ (also see Eq.~\eqref{eq:marginal_PP}),
      \begin{equation}
        p_t(\bm{x}) = \int p_t(\bm{x}|\bm{s}) q^{\circ}(\bm{s})  {\rm d} \bm{s}.
      \end{equation}
      where each conditional probability path $p_t(\bm{x}|\bm{s})$ arises from a corresponding conditional vector field $\bm{u}_t(\bm{x}|\bm{s})$, i.e., satisfying the continuity equation,
      \begin{equation}\label{eq_continuity_cond}
        \frac{\partial  p_t(\bm{x}|\bm{s})}{\partial t} = -\nabla \cdot [p_t(\bm{x}|\bm{s})\bm{u}_t(\bm{x}|\bm{s})]
      \end{equation}
      with the boundary conditions $p_0(\bm{x}_0|\bm{s}) = q_0(\bm{x}_0|\bm{s})$ and $p_1(\bm{x}_1|\bm{s}) = q_1(\bm{x}_0|\bm{s})$. Moreover, it holds the boundary marginal source and target distributions, i.e.,
      \begin{equation}
        \begin{aligned}
          p_i(\bm{x}) &= \int p_i(\bm{x}|\bm{s}) q^{\circ}(\bm{s})  {\rm d} \bm{s}\\
          &=  \int q_i(\bm{x}|\bm{s}) q^{\circ}(\bm{s})  {\rm d} \bm{s} \\
          &= q_i(\bm{x}),
        \end{aligned}
      \end{equation}
      where $i \in \{0, 1\}$. Therefore, one can sample $p_t(\bm{x})$ through the above three steps.
    \end{proofE}
    \begin{remark}
      In this work, we are interested in the simple switching mechanism where the $q^{\circ}(\bm{s})$ and $q_0(\bm{x}_0|\bm{s})$ (resp., $q_1(\bm{x}_1|\bm{s})$) are both easily sampled, which will be presented in the Subsection~\ref{subsec:SM}. Additionally, these conditional vector fields $\bm{u}_t(\bm{x}|\bm{s})$, in turn, collectively generate a marginal vector field, obtained by ``marginalizing'' over them as follows:
      \begin{equation}
        \label{eq:marginal_VF}
        \bm{u}_t(\bm{x}) := \int \bm{u}_t(\bm{x}|\bm{s}) \frac{p_t(\bm{x}|\bm{s}) q^{\circ}(\bm{s})}{ p_t(\bm{x})}  {\rm d} \bm{s},
      \end{equation}
      where $p_t(\bm{x})>0$ for all $t$ and $\bm{x}$. Crucially, as pointed out in the existing studies \cite{lipman2022flow, pooladian2023multisample, tong2023conditional, tong2023improving}, the marginal vector field~\eqref{eq:marginal_VF} actually generates the marginal probability path~\eqref{eq:marginal_PP}. However, using a single ODE to solve the transportation problem may inevitably encounter the singularity problem due to the inherent (joint) heterogeneity of the source and/or target distributions as discussed in the Section~\ref{sec:limitations}.
    \end{remark}
    \subsection{Training Objective}
    To mitigate the issue of singularity,  our study aims to directly approximate the conditional vector field $\bm{u}_t(\bm{x}|\bm{s})$ by the learnable one $\bm{v}_t(\bm{x};\bm{\theta} |\bm{s} )$ using the following SFM objective:
    \begin{equation}
      \label{obj_SFM}
      \mathcal{L}_{\text{SFM}}(\bm{\theta})=\mathbb{E}_{t, q^{\circ}(\bm{s}), p_t(\bm{x}|\bm{s})}
      \| \bm{v}_t(\bm{x};\bm{\theta} |\bm{s} )
      - \bm{u}_t(\bm{x}|\bm{s})\|^2.
    \end{equation}
    Simply put, the SFM loss~\eqref{obj_SFM} regresses the conditional vector field $\bm{u}_t(\bm{x}|\bm{s})$ with a neural network $\bm{v}_t(\bm{x};\bm{\theta} |\bm{s} )$ via consistently sharing the parameter vector $\bm{\theta}$ across all switching signals $\bm{s}$. Upon minimizing the SFM loss to zero, an efficient sampling mechanism is enabled by the replacement of $\bm{u}_t(\bm{x}|\bm{s})$ with $\bm{v}_t(\bm{x};\bm{\theta} |\bm{s})$ as proposed in Proposition~\ref{thm:Switching_ODEs}.
    
    However, akin to the FM~\eqref{obj_FM}, the SFM objective~\eqref{obj_SFM} becomes intractable in the absence of prior knowledge regarding the appropriate forms of $p_t(\bm{x}|\bm{s})$ and $\bm{u}_t(\bm{x}|\bm{s})$. To address this issue, similar to the CFM~\eqref{obj_CFM}, we further introduce a latent variable $\bm{z}$, and by marginalizing the conditional probability paths over $q(\bm{z}|\bm{s})$, we have the marginal probability path condition on $\bm{s}$,
    \begin{equation}
      \label{eq:cond_marignal_prob_path}
      p_t(\bm{x}|\bm{s}) = \int p_t(\bm{x}|\bm{z}, \bm{s}) q(\bm{z}|\bm{s}) {\rm d} \bm{z}.
    \end{equation}
    Akin to the marginal vector field~\eqref{eq:marginal_VF}, we can also obtain the marginal vector field given $\bm{s}$, i.e., $\bm{u}_t(\bm{x}|\bm{s})$, by marginalizing over the conditional vector fields $\bm{u}_t(\bm{x}|\bm{z},\bm{s})$ in the following sense,
    \begin{equation}
      \label{eq:cond_marginal_VF}
      \bm{u}_t(\bm{x}|\bm{s}) := \int \bm{u}_t(\bm{x}|\bm{z},\bm{s}) \frac{p_t(\bm{x}|\bm{z},\bm{s}) q(\bm{z}|\bm{s})}{ p_t(\bm{x}|\bm{s})}  {\rm d} \bm{z},
    \end{equation}
    where $\bm{u}_t(\bm{x}|\bm{z},\bm{s})$ is the conditional vector field that generates $p_t(\bm{x}|\bm{z},\bm{s})$, yielding the following result.
    \begin{propositionE}[][end,restate]\label{thm:id_gradient}
      Given the switching signal $s$, the vector field $\bm{u}_t(\bm{x}|\bm{s})$ in Eq.~\eqref{eq:cond_marginal_VF} generates the probability path $p_t(\bm{x}|\bm{s})$ in Eq.~\eqref{eq:cond_marignal_prob_path}.
    \end{propositionE}
    \begin{proofE}
      The proof is adapted from \citet{lipman2022flow, tong2023conditional, tong2023improving}.
      
      Since $\bm{u}_t(\bm{x}|\bm{z},\bm{s})$ is the conditional vector field that generates $p_t(\bm{x}|\bm{z},\bm{s})$, it means that given the switching signal $s$ and the latent variable $\bm{z}$, $\bm{u}_t(\bm{x}|\bm{z},\bm{s})$ and $p_t(\bm{x}|\bm{z},\bm{s})$ satisfy the continuity equation:
      \begin{equation}
        \frac{\partial  p_t(\bm{x}|\bm{z},\bm{s})}{\partial t} = -\nabla \cdot [p_t(\bm{x}|\bm{z},\bm{s})\bm{u}_t(\bm{x}|\bm{z},\bm{s})].
      \end{equation}
      Next, we check that given the switching signal $s$, $p_t(\bm{x}|\bm{s})$ and $\bm{u}_t(\bm{x}|\bm{s})$ satisfy the continuity equation:
      \begin{equation}
        \begin{aligned}
          \frac{\partial  p_t(\bm{x}|\bm{s})}{\partial t} &= \frac{\partial }{\partial t} \int  p_t(\bm{x}|\bm{z}, \bm{s}) q(\bm{z}|\bm{s}) {\rm d} \bm{z}   \\
          &=  \int \left[\frac{\partial }{\partial t} p_t(\bm{x}|\bm{z}, \bm{s}) \right] q(\bm{z}|\bm{s}) {\rm d} \bm{z}   \\
          &=  -\int \left\{ \nabla \cdot [p_t(\bm{x}|\bm{z}, \bm{s})\bm{u}_t(\bm{x}|\bm{z}, \bm{s})] \right\} q(\bm{z}|\bm{s}) {\rm d} \bm{z}   \\
          &=  -\nabla \cdot \int p_t(\bm{x}|\bm{z}, \bm{s})\bm{u}_t(\bm{x}|\bm{z}, \bm{s}) q(\bm{z}|\bm{s}) {\rm d} \bm{z}   \\
          &= -\nabla \cdot \left[p_t(\bm{x}|\bm{s}) \int \bm{u}_t(\bm{x}|\bm{z},\bm{s}) \frac{p_t(\bm{x}|\bm{z},\bm{s}) q(\bm{z}|\bm{s})}{ p_t(\bm{x}|\bm{s})}  {\rm d} \bm{z}  \right]\\
          &:=-\nabla \cdot [p_t(\bm{x}|\bm{s})\bm{u}_t(\bm{x}|\bm{s})]
        \end{aligned}
      \end{equation}
      where we assume that the functions being integrated satisfy the regularity conditions for exchanging integration and differentiation.
    \end{proofE}
    Similar to the CFM~\eqref{obj_CFM}, we then consider the Switching Conditional FM (SCFM) objective:
    \begin{equation}
      \label{obj_SCFM}
      \begin{aligned}
        \mathcal{L}_{\text{SCFM}}(\bm{\theta})=\mathbb{E}_{t, q^{\circ}(\bm{s}), q(\bm{z}|\bm{s}), p_t(\bm{x}|\bm{z},\bm{s})}
        & \| \bm{v}_t(\bm{x};\bm{\theta} |\bm{s} ) \\
        & - \bm{u}_t(\bm{x}|\bm{z}, \bm{s})\|^2.
      \end{aligned}
    \end{equation}
    Then, we have the following result.
    \begin{propositionE}[][end,restate]\label{thm:id_gradient}
      Assuming that $p_t(\bm{x}|\bm{s})>0$ for all $\bm{x}\in\mathbb{R}^d$ and $t\in[0, 1]$, then, up to a constant independent of $\bm{\theta}$,  $\mathcal{L}_{\text{SCFM}}(\bm{\theta})$ and $\mathcal{L}_{\text{SFM}}(\bm{\theta})$ are equal. Hence, $\nabla_{\bm{\theta}} \mathcal{L}_{\text{SCFM}}(\bm{\theta}) = \nabla_{\bm{\theta}} \mathcal{L}_{\text{SFM}}(\bm{\theta})$.
    \end{propositionE}
    \begin{proofE}
      The proof is adapted from \citet{lipman2022flow, tong2023conditional, tong2023improving}.
      
      To ensure the existence of all integrals and to allow the changing of integration order (by Fubini's Theorem), we assume that $q(\bm{x}|\bm{s})$, $p_t(\bm{x}|\bm{z}, \bm{s})$ are decreasing to zero at sufficient speed as $\|\bm{x}\|\to \infty$ and that $\bm{u}_t, \bm{v}_t, \nabla_{\bm{\theta}} \bm{v}_t$ are bounded. Since $t$ and $\bm{s}$ are sampled from $\mathcal{U}(0, 1)$ and $q^{\circ}(\bm{s})$, respectively, where both are independent of $\bm{\theta}$,  in the following $t$ and $\bm{s}$ are both fixed. By the bilinearity of the Euclidean norm and since $\bm{u}_t$ is independent of $\bm{\theta}$, we have
      \begin{equation}\label{eq:SFM_appd}
        \begin{aligned}
          &\nabla_{\bm{\theta}} \mathbb{E}_{p_t(\bm{x}|\bm{s})} \|\bm{v}_t(\bm{x};\bm{\theta} |\bm{s}) - \bm{u}_t(\bm{x}|\bm{s})\|^2 \\
          =& \nabla_{\bm{\theta}}  \mathbb{E}_{p_t(\bm{x}|\bm{s})} \left (\|\bm{v}_t(\bm{x};\bm{\theta} |\bm{s})\|^2 - 2 \left \langle  \bm{v}_t(\bm{x};\bm{\theta} |\bm{s}), \bm{u}_t(\bm{x}|\bm{s}) \right \rangle + \|\bm{u}_t(\bm{x}|\bm{s})\|^2  \right )\\
          =& \nabla_{\bm{\theta}}  \mathbb{E}_{p_t(\bm{x}|\bm{s})} \left (\|\bm{v}_t(\bm{x};\bm{\theta} |\bm{s})\|^2 - 2 \left \langle  \bm{v}_t(\bm{x};\bm{\theta} |\bm{s}), \bm{u}_t(\bm{x}|\bm{s}) \right \rangle \right ),
        \end{aligned}
      \end{equation}
      and
      \begin{equation}\label{eq:SCFM_appd}
        \begin{aligned}
          &\nabla_{\bm{\theta}} \mathbb{E}_{q(\bm{z}|\bm{s}), p_t(\bm{x}|\bm{z},\bm{s})}  \|\bm{v}_t(\bm{x};\bm{\theta} |\bm{s}) - \bm{u}_t(\bm{x}|\bm{z},\bm{s})\|^2 \\
          =& \nabla_{\bm{\theta}} \mathbb{E}_{q(\bm{z}|\bm{s}), p_t(\bm{x}|\bm{z},\bm{s})} \left (\|\bm{v}_t(\bm{x};\bm{\theta} |\bm{s})\|^2 - 2 \left \langle  \bm{v}_t(\bm{x};\bm{\theta} |\bm{s}), \bm{u}_t(\bm{x}|\bm{z},\bm{s}) \right \rangle + \|\bm{u}_t(\bm{x}|\bm{z},\bm{s})\|^2 \right ) \\
          =& \nabla_{\bm{\theta}} \mathbb{E}_{q(\bm{z}|\bm{s}), p_t(\bm{x}|\bm{z},\bm{s})}  \left (\|\bm{v}_t(\bm{x};\bm{\theta} |\bm{s})\|^2 - 2 \left \langle  \bm{v}_t(\bm{x};\bm{\theta} |\bm{s}), \bm{u}_t(\bm{x}|\bm{z},\bm{s}) \right \rangle \right ).
        \end{aligned}
      \end{equation}
      Next,
      \begin{equation}
        \begin{aligned}
          & \mathbb{E}_{p_t(\bm{x}|\bm{s})} \|\bm{v}_t(\bm{x};\bm{\theta} |\bm{s})\|^2 \\
          =& \int \|\bm{v}_t(\bm{x};\bm{\theta} |\bm{s})\|^2 p_t(\bm{x}|\bm{s}) {\rm d}\bm{x} \\
          =& \iint \|\bm{v}_t(\bm{x};\bm{\theta} |\bm{s})\|^2 p_t(\bm{x}|\bm{z},\bm{s}) q(\bm{z}|\bm{s}) {\rm d}\bm{z}  {\rm d}\bm{x} \\
          =& \mathbb{E}_{q(\bm{z}|\bm{s}), p_t(\bm{x}|\bm{z},\bm{s})} \|\bm{v}_t(\bm{x};\bm{\theta} |\bm{s})\|^2.
        \end{aligned}
      \end{equation}
      Finally,
      \begin{equation}
        \begin{aligned}
          \mathbb{E}_{p_t(\bm{x}|\bm{s})} \left \langle  \bm{v}_t(\bm{x};\bm{\theta} |\bm{s}), \bm{u}_t(\bm{x}|\bm{s}) \right \rangle
          &= \int \left \langle  \bm{v}_t(\bm{x};\bm{\theta} |\bm{s}), \frac{\int \bm{u}_t(\bm{x}|\bm{z},\bm{s}) p_t(\bm{x} | \bm{z}, \bm{s}) q(\bm{z}|\bm{s}) {\rm d}\bm{z}}{p_t(\bm{x}|\bm{s})} \right \rangle p_t(\bm{x}|\bm{s}) {\rm d}\bm{x} \\
          &= \int \left \langle  \bm{v}_t(\bm{x};\bm{\theta} |\bm{s}), \int \bm{u}_t(\bm{x}|\bm{z},\bm{s}) p_t(\bm{x} | \bm{z}, \bm{s}) q(\bm{z}|\bm{s}) {\rm d}\bm{z}\right \rangle {\rm d}\bm{x} \\
          &= \iint \left \langle  \bm{v}_t(\bm{x};\bm{\theta} |\bm{s}), \bm{u}_t(\bm{x}|\bm{z},\bm{s}) \right \rangle p_t(\bm{x} | \bm{z}, \bm{s}) q(\bm{z}|\bm{s}) {\rm d}\bm{z} {\rm d}\bm{x} \\
          &= \mathbb{E}_{q(\bm{z}|\bm{s}), p_t(\bm{x} | \bm{z}, \bm{s})} \left \langle  \bm{v}_t(\bm{x};\bm{\theta} |\bm{s}), \bm{u}_t(\bm{x}|\bm{z},\bm{s}) \right \rangle.
        \end{aligned}
      \end{equation}
      Therefore, Eq.~\eqref{eq:SFM_appd} is always equal to Eq.~\eqref{eq:SCFM_appd} for any $\bm{s}$ and $\bm{z}$, implying that $\nabla_{\bm{\theta}} \mathcal{L}_{\text{SCFM}}(\bm{\theta}) = \nabla_{\bm{\theta}} \mathcal{L}_{\text{SFM}}(\bm{\theta})$. The proof is complete.
    \end{proofE}
    \begin{remark}
      The above result is actually the same as studied in \citet{lipman2022flow, pooladian2023multisample, tong2023conditional, tong2023improving} if we consider the switching signal $s$ is a dumb variable, i.e., $q_i(\bm{x}|\bm{s})=q_i(\bm{x}), i\in\{0, 1\}$, $q(\bm{z}|\bm{s})=q(\bm{z})$, $\bm{u}_t(\bm{x}|\bm{z}, \bm{s}) = \bm{u}_t(\bm{x}|\bm{z})$, and $\bm{v}_t(\bm{x};\bm{\theta}|\bm{s}) = \bm{v}_t(\bm{x};\bm{\theta})$.
    \end{remark}
    The SCFM objective~\eqref{obj_SCFM} is useful when the vector field $\bm{u}_t(\bm{x}|\bm{s})$ is intractable but the conditional vector field $\bm{u}_t(\bm{x}|\bm{z}, \bm{s})$ is simple even in a closed form.
    \subsection{Coupling}
    As delineated in Eqs.~\eqref{eq:ind_coupling}-\eqref{eq:ind_VF}, one can also choose $q(\bm{z}|\bm{s})$ as an independent coupling condition on $\bm{s}$, i.e.,
    \begin{equation}
      \label{eq:SCFM_ind_coupling}
      q(\bm{z}|\bm{s}):=q(\bm{x}_0, \bm{x}_1|\bm{s})=q_0(\bm{x}_0|\bm{s})q_1(\bm{x}_1|\bm{s}),
    \end{equation}
    resulting in the linear interpolation $\bm{x}(t)$ and the constant speed vector field condition on both $\bm{z}$ and $\bm{s}$:
    \begin{equation}
      \label{eq:SCFM_ind_VF}
      \bm{x}(t)= (1-t)\bm{x}_0 + t \bm{x}_1, \quad \bm{u}_t(\bm{x}|\bm{z},  \bm{s})=\bm{x}_1 - \bm{x}_0.
    \end{equation}
    In addition, another choice of $q(\bm{z}|\bm{s})$ is the optimal coupling \cite{pooladian2023multisample, tong2023conditional, tong2023improving} in terms of the squared $2$-Wasserstein distance condition on $\bm{s}$, namely,
    \begin{equation}
      q(\bm{z}|\bm{s}) := q^{*}(\bm{x}_0, \bm{x}_1|\bm{s}),
    \end{equation}
    where $\bm{z}$ represents a pair of points $\bm{x}_0$ and $\bm{x}_1$. Contrary to independently sampling them from their conditional distributions~\eqref{eq:SCFM_ind_coupling}, these points are jointly sampled in accordance with the optimal coupling $q^{*}(\bm{x}_0, \bm{x}_1|\bm{s})$ condition on $\bm{s}$. Here, we also use the simple vector field $\bm{u}_t(\bm{x}|\bm{z}, \bm{s})$ as defined in Eq.~\eqref{eq:SCFM_ind_VF} in the SCFM objective~\eqref{obj_SCFM}. We then propose the following result.
    \begin{propositionE}[][end,restate]\label{thm:dyn_OT}
      Consider the optimal coupling $q^{*}(\bm{x}_0, \bm{x}_1|\bm{s})$ and the vector field $\bm{u}_t(\bm{x}|\bm{z}, \bm{s})$ as defined in Eq.~\eqref{eq:SCFM_ind_VF}, then the optimal vector field $\bm{v}_t(\bm{x};\bm{\theta} |\bm{s})$ in Eq.~\eqref{obj_SCFM} solves the dynamic optimal transport problem~\eqref{eq:dyn_OT} (condition on $\bm{s}$) between $q_0(\bm{x_0}|\bm{s})$ and $q_1(\bm{x_1}|\bm{s})$.
    \end{propositionE}
    \begin{proofE}
      Here, we assume that given the switching signal $\bm{s}$,  the source and target distributions condition on $\bm{s}$ satisfy the regularity conditions such that by Brenier's theorem \cite{brenier1991polar}, there is a unique optimal Monge coupling between $q_0(\bm{x_0}|\bm{s})$ and $q_1(\bm{x_1}|\bm{s})$.
      
      Under the optimal coupling $q^{*}(\bm{z}|\bm{s}):=q^{*}(\bm{x}_0, \bm{x}_1|\bm{s})$, it induces a unique optimal Monge transport map $T(\cdot|\bm{s})$, which can be represented by the gradient of some convex function $\Phi(\cdot|\bm{s})$, i.e.,
      \begin{equation}
        \bm{x}_1 = T(\bm{x}_0|\bm{s}) = \nabla \Phi(\bm{x}_0|\bm{s}),
      \end{equation}
      where $\bm{x}_0\sim q_0(\bm{x_0}|\bm{s})$ and $\bm{x}_1\sim q_1(\bm{x_1}|\bm{s})$. In addition, we can construct the conditional probability path or equivalently the flow map as:
      \begin{equation}
        \bm{\phi}_t(\bm{x}_0|\bm{s}) = \bm{x}_0 + t [T(\bm{x}_0|\bm{s}) - \bm{x}_0],
      \end{equation}
      with the associated vector field:
      \begin{equation}\label{eq:opt_monge_VF}
        \bm{u}_t(\bm{x}|\bm{s}) = T(\bm{x}_0|\bm{s}) - \bm{x}_0.
      \end{equation}
      Therefore, the optimal vector field $\bm{v}_t(\bm{x};\bm{\theta} |\bm{s})$ in Eq.~\eqref{obj_SCFM} (i.e., equal to the above Eq.~\eqref{eq:opt_monge_VF}), solves the dynamic optimal transport problem~\eqref{eq:dyn_OT} (condition on $\bm{s}$) between $q_0(\bm{x_0}|\bm{s})$ and $q_1(\bm{x_1}|\bm{s})$.
    \end{proofE}
    \begin{remark}
      If we consider the switching signal $s$ as a dumb variable, then the above result is actually the same as studied in \citet{tong2023conditional, tong2023improving}. However, using a single ODE to solve the dynamic optimal transport problem~\eqref{eq:dyn_OT} may not satisfy certain regularity assumptions. For example, the support of a distribution needs to be connected, which is often not the case in reality as discussed in the Section~\ref{sec:limitations}. On the contrary, to eliminate the singularities, the SFM uses multiple ODEs to solve it, which is conditionally or locally optimal (see the next subsection).
    \end{remark}
    
    In practice, this optimal coupling can be approximated by addressing optimal transport problems within a given data batch \cite{pooladian2023multisample, tong2023conditional, tong2023improving}. Specifically, for each data batch $\left\{ \bm{x}_0^{(k)} \right\}_{k=1}^m \sim q_0(\bm{x}_0|\bm{s})$ and $\left\{ \bm{x}_1^{(k)} \right\}_{k=1}^m \sim q_1(\bm{x}_1|\bm{s})$, the optimal transport problem~\eqref{eq:OT} condition on $\bm{s}$ for the discrete case can be exactly and efficiently resolved using standard solvers, such as the \texttt{POT} \citep[Python Optimal Transport]{flamary2021pot}.
    
    Here, we call these two methods independent SFM (I-SFM) and optimal transport SFM (OT-SFM).
    
    \subsection{Switching Mechanism}\label{subsec:SM}
    Motivated by our observations and theories, we focus on constructing a simple and efficient switching mechanism such that the $q^{\circ}(\bm{s})$ and $q_0(\bm{x}_0|\bm{s})$ (resp., $q_1(\bm{x}_1|\bm{s})$) are both easily sampled for the general source and target distributions. One possible way is to employ the classic clustering methods to partition the empirical source (resp., target) dataset $\bm{X}_0 \sim q_0(\bm{x}_0)$ (resp.,  $\bm{X}_1 \sim q_1(\bm{x}_1)$) into $K_0$ (resp., $K_1$) sets, i.e., $\bm{X}_0^{(1)},...,\bm{X}_0^{(K_0)}$ (resp., $\bm{X}_1^{(1)},...,\bm{X}_1^{(K_1)}$). In addition, we assign each set $\bm{X}_0^{(i)}$ (resp., $\bm{X}_1^{(j)}$) a label $y_0^{(i)}$ (resp., $y_1^{(j)}$) and its weight or mass $\rho_0^{(i)}=|\bm{X}_0^{(i)}|/|\bm{X}_0|$ (resp., $\rho_1^{(j)}=|\bm{X}_1^{(j)}|/|\bm{X}_1|$).
    
    \textbf{General setup.} We then construct the switching mechanism in the following manner:
    \begin{enumerate}%[leftmargin=*, topsep=0mm, itemsep=0mm, parsep=0mm]
      \item $\bm{s}$ is a discrete variable, defined as $\bm{s} := (y_0, y_1)\in\left\{(y_0^{(i)}, y_1^{(j)})|i=1,...,K_0, j=1,...,K_1\right\}$;
      \item  $q^{\circ}(\bm{s}): = q^{\circ}(y_0, y_1)$ is a discrete (joint) distribution, defined as a coupling matrix $P$, satisfying the conservation of mass ($K_0 + K_1$ equality constraints),
      \begin{equation}\label{eq:coupling_matrix}
        \sum_{j=1}^{K_1} P(i, j) = \rho_0^{(i)}, \quad \sum_{i=1}^{K_0} P(i, j) = \rho_1^{(j)},
      \end{equation}
      where the element $P(i, j)\geq 0$ describes the amount of mass flowing from the bin $i$ (or the set $\bm{X}_0^{(i)}$) towards the bin $j$ (or the set $\bm{X}_1^{(j)}$);
      \item $q_0(\bm{x_0}|\bm{s})$ (resp., $q_1(\bm{x_1}|\bm{s})$) is an empirical data distribution available as finite samples, i.e., $\bm{X}_0^{(i)}$ (resp., $\bm{X}_1^{(j)}$).
    \end{enumerate}
    By choosing the different coupling matrix $P$, we induce the different switching signal distributions $q^{\circ}(\bm{s})=q^{\circ}(y_0, y_1)$.
    
    \textbf{Optimal transport setup.} If  $P^{*}$ is the solution of the discrete Kantorovich's optimal transport problem, i.e.,
    \begin{equation}
      P^{*} = \arg \min_{P} \langle C, P \rangle: = \sum_{i, j}C(i, j)P(i, j),
    \end{equation}
    where $C(i, j)$ is the cost of moving a single unit from bin $i$ to bin $j$, then $P^{*}$ has the following property.
    \begin{propositionE}[Extremal solutions \cite{peyre2019computational}][end,restate]\label{thm:extrem_point}
      $P^{*}$ cannot have more than $K_0 + K_1 - 1$ nonzero entries, i.e., $|\{(y_0^{(i)}, y_1^{(j)})|P^{*}(i, j)>0\}| \leq K_0 + K_1 - 1 $.
    \end{propositionE}
    \begin{proofE}
      See Proposition 3.4 in \cite{peyre2019computational}. The extremal property is illustrated in Fig.~\ref{fig_extrem_point}.
      \begin{figure*}[htb]
        \begin{center}
          \centering
          \subfigure{\label{fig_extrem_point_a}}
          \subfigure{\label{fig_extrem_point_b}}
          \includegraphics[width=0.4\textwidth]{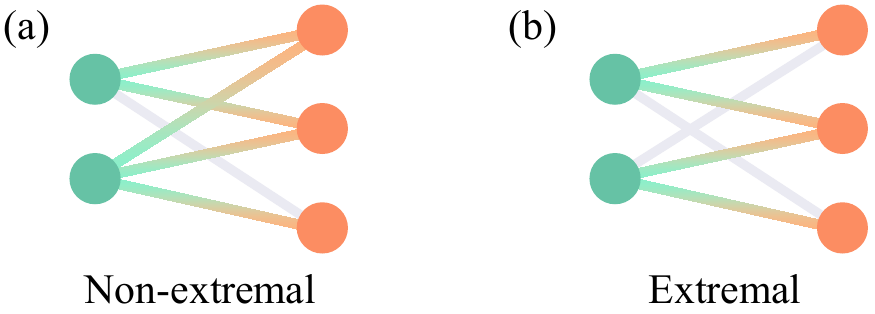}
          \caption{Illustration of the extremal property in Proposition~\ref{thm:extrem_point}.}
          \label{fig_extrem_point}
        \end{center}
      \end{figure*}
    \end{proofE}
    \begin{remark}
      In practice, $C(i, j)$ is simply assigned as either a uniform constant or as a function representing the appropriate distance between the sets $\bm{X}_0^{(i)}$ and $\bm{X}_1^{(j)}$. Furthermore, according to Proposition~\ref{thm:extrem_point}, it is possible to reduce the number of states $\bm{s}=(y_0, y_1)$ from a higher-order complexity of $K_0K_1$ to a linear complexity of $K_0 + K_1 - 1$. \end{remark}
      
      \begin{figure*}[t]
        \begin{center}
        \centerline{\includegraphics[width=0.98\textwidth]{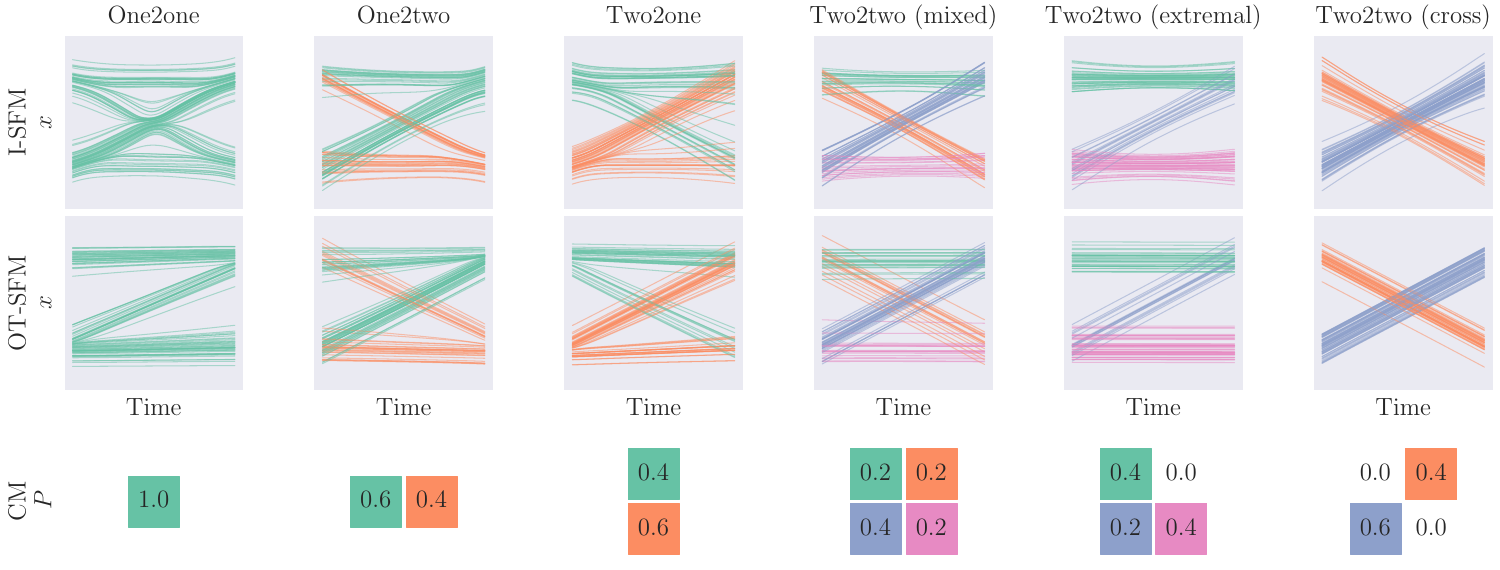}}
          % \vskip -0.1in
          \caption{Trajectories of the I-SFM and the OT-SFM on $2$-d Gaussian mixtures under different coupling matrices $P$ (from left to right). Particularly, in the first column (``one2one'' coupling), the I-SFM and the OT-SFM are the I-CFM and OT-CFM, respectively.}
          \label{fig_2gaussians}
        \end{center}
        % \vskip -0.4in
      \end{figure*}
      
      \section{Related Works}
      \label{sec:related_works}
      \textbf{Switched systems.}  Mathematically, switched systems are hybrid dynamical systems that consist of a family of subsystems and a rule that determines the switching between them \cite{liberzon1999basic, liberzon1999stability, daafouz2002stability, liberzon2003switching}. Typically, the rules can be largely divided into state-dependent and time-dependent switching. It should be pointed out that these switchings occur during the evolution process of a system. In contrast, as shown in Proposition~\ref{thm:Switching_ODEs}, our switching mechanism involves randomly sampling a system, and then keeping it unchanged over time.
      
      \textbf{Conditional generation.}  Class-conditional generation is a common and important task, whose goal is to generate a sample that belongs to a specified class of the target distributions via incorporating the class label into their models \cite{van2016conditional, nguyen2017plug, odena2017conditional, ho2022cascaded}. This can be regarded as a special case of our framework by setting the switching signal as the target label. Since there is a lot of literature on this topic and our goal is to theoretically elucidate the limitations of FM and to eliminate singularities raised by using a single ODE, it is beyond the scope of this paper to have a complete review of the existing literature.
      
      \section{Experiments}
      \label{sec:experiments}
      \textbf{Synthetic datasets.} Figure~\ref{fig_2gaussians} shows the proposed I-SFM and OT-SFM on transporting an $1$-d Gaussian mixture ($2$-modes) to another. It is observed that an appropriate switching rule can eliminate the singularity raised from the heterogeneities of source and target distributions, leading to better regularity. In other words, when the data is sampled near the singularity region, it is inevitable that both the I-SFM and OT-SFM  tend to perform poorly, but our framework is capable of achieving relatively good results. In addition, OT-SFM leads to a straighter flow than I-SFM.
      
      Figure~\ref{fig_Infinite_trained} shows the learned flows of the I-SFM and the OT-SFM on the example of the infinite number of singular points under the optimal coupling in Proposition~\ref{thm:infinite}.
      \begin{figure}[htb]
        % \vskip -0.3in
        \begin{center}
          \centering
          \subfigure{\label{fig_Infinite_a}}
          \subfigure{\label{fig_Infinite_b}}
          \includegraphics[width=0.48\textwidth]{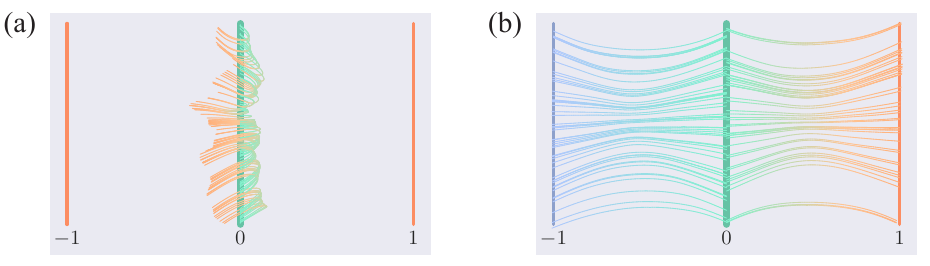}
          % \vskip -0.1in
          \caption{The learned flows of the I-CFM (a) and the I-SFM (one2two) (b) on the example in Proposition~\ref{thm:infinite}.}
          \label{fig_Infinite_trained}
        \end{center}
        % \vskip -0.2in
      \end{figure}
      
      \textbf{CIFAR-10 dataset.} Table~\ref{tab:cifar10} shows the image generation results of our SFM variants on the CIFAR-10 dataset. In contrast with the existing generative models, we, here, consider a general source distribution, a Gaussian
      mixture with two modes, instead of a standard Gaussian distribution. Therefore, we display the results of the I-CFM and OT-CFM as the baselines. Crucially, it is observed that the I-CFM performs poorly on this task due to the mode separation of the source distribution, while it worked well for the standard Gaussian distribution \cite{lipman2022flow, liu2022flow, tong2023conditional, tong2023improving}. In addition, the I-SFM (one2ten) and OT-SFM (one2ten) perform poorly as well, as they all treat the support of the source distribution as one mode. Other SFM variants that explicitly separate the two modes of the source distribution, all perform well even better than the OT-CFM. We note that the OT-SFM did not perform as well as expected in comparison to the I-SFM. We attribute the reason to the switching mechanism that has already alleviated singularities induced by mode separation.
      
      \begin{table}[htb]
        \centering
        % \vskip -0.1in
        \caption{FID results of CFM and SFM on the CIFAR-10 dataset.}
        \vspace*{1em}
        \label{tab:cifar10}
        \resizebox{\linewidth}{!}{    \begin{tabular}{@{}l|ccccccc}
        \toprule
        \toprule
        NFE  & 6 & 8 & 10 & 20 & 40 &  Adap.\\
        \midrule
        I-CFM (I-SFM, one2one) & $144.52$ & $130.49$ & $122.44$ & $106.11$ & $99.19$ & $94.55$\\
        OT-CFM (OT-SFM, one2one) & $176.80$ & $111.09$ & $76.41$ & $26.15$ & $10.90$ & $4.91$\\
        \midrule
        I-SFM (one2ten) & $\mathbf{109.24}$ & $98.47$ & $93.48$ & $83.41$ & $78.33$ & $75.06$\\
        OT-SFM (one2ten) & $122.74$ & $104.19$ & $93.04$ & $73.47$ & $63.94$ & $59.72$\\
        I-SFM (two2one) & $177.99$ & $115.05$ & $78.46$ & $23.91$ & $9.18$ & $5.21$\\
        OT-SFM (two2one) & $185.44$ & $121.21$ & $84.32$ & $28.18$ & $11.11$ & $5.64$\\
        I-SFM (two2ten, mixed) & $132.41$ & $75.83$ & $\mathbf{49.53}$ & $\mathbf{15.60}$ & $\mathbf{6.98}$ & $4.27$\\
        OT-SFM (two2ten, mixed) & $133.27$ & $76.31$ & $49.69$ & $15.50$ & $7.24$ & $4.39$\\
        I-SFM (two2ten, extremal) & $128.55$ & $\mathbf{75.11}$ & $50.12$ & $17.14$ & $8.39$ & $\mathbf{4.22}$\\
        OT-SFM (two2ten, extremal) & $149.50$ & $88.33$ & $58.25$ & $18.59$ & $8.86$ & $4.40$\\
        \bottomrule
        \bottomrule
      \end{tabular}
      }
      % \vskip -0.2in
    \end{table}
    
    \section{Conclusion}
    \label{sec:conclusion}
    In this article, we highlighted and analyzed the limitations of FM, where using a single ODE for generative modeling may inevitably encounter the singularity problem due to the inherent (joint) heterogeneity of the source and/or target distributions. To eliminate singularities, we proposed SFM via switching multiple ODEs, even allowing the intersection of trajectories from distinct ODEs while it is impossible for a single ODE. In addition, a simple and efficient switching mechanism was constructed for effective training and inference. From an orthogonal perspective, our framework can seamlessly integrate with the existing advanced techniques, such as minibatch optimal transport, to further enhance the straightness of each flow. We also demonstrated the exceptional efficacy of the proposed framework by using synthetic and real-world datasets. We hope that our findings and proposed framework can contribute to the advancement of the field of generative modeling.

    \section*{Impact Statement}
    This paper presents work whose goal is to advance the field of Machine Learning. There are many potential societal consequences of our work, none of which we feel must be specifically highlighted here.

    \section*{Acknowledgements}
    Q. Zhu is supported by the China Postdoctoral Science Foundation (No. 2022M720817), by the Shanghai Postdoctoral Excellence Program (No. 2021091), and by the STCSM (Nos. 21511100200, 22ZR1407300, 22dz1200502, and 23YF1402500). W. Lin is supported by the NSFC (Grant No. 11925103), by the STCSM (Grants No. 22JC1402500 and No. 22JC1401402), and by the SMEC (Grant No. 2023ZKZD04). The computational work presented in this article is supported by the CFFF platform of Fudan University.
    
    \nocite{langley00}
    
    \bibliography{example_paper}
    \bibliographystyle{icml2024}
    
    \newpage
    \appendix
    \onecolumn
    
    \textbf{Roadmap.} The structure of the appendix is outlined as follows:
    \begin{itemize}
      \item Appendix \ref{sec:appd_1} presents the existing theoretical results on ODEs and optimal transport.
      \item Appendix \ref{sec:proofs} presents the proofs of our theoretical results.
      \item Appendix \ref{sec:appd_additional_results} presents additional results and discussions from the perspectives of the methodology, algorithm, and experiment.
      \item Appendix \ref{sec:appd_experiment} presents the experimental details for all experiments conducted in the work.
    \end{itemize}
    
    \section{Existing theoretical results}
    \label{sec:appd_1}
    In this section, before presenting the proofs of our theoretical results in the main text as well as the additional results in the appendix, we first review the existing theoretical results on ODEs.
    
    Throughout the section, we consider the IVP:
    \begin{equation}
      \label{appd: eq_ODE}
      \begin{aligned}
        \frac{{\rm d} \bm{x}(t)}{{\rm d} t} &= \bm{u}_t(\bm{x}), \quad t\in[0, 1], \\
        \bm{x}(0) &=\bm{x}_0, \\      \end{aligned}
      \end{equation}
      where $\bm{u}_t(\bm{x}): [0, 1]\times \mathbb{R}^d \rightarrow \mathbb{R}^d$ is a smooth vector field with a bounded Lipschitz constant $L:=\|\bm{u}_t(\bm{x})\|_{\text{Lip}}$. Then, we have the following existence and uniqueness theorem.
      
      \subsection{Properties of ODEs}
      \begin{theorem}[Global existence and uniqueness, \cite{ahmad2015textbook}]
        Suppose that $\bm{x}\in \mathbb{R}^d$, $t\in[0, 1]$, and $\bm{u}_t(\bm{x})$ is continuous and globally lipschitzian in $\mathbb{R}^d$ with respect to $\bm{x}$, then the solution of \eqref{appd: eq_ODE} is unique and defined on all $t\in[0, 1]$.
      \end{theorem}
      
      \begin{theorem}[Non-intersecting trajectories, \cite{coddington1956theory, younes2010shapes, dupont2019augmented}]
        \label{appd:thm_noncross}
        Let $\bar{\bm{x}}(t)$ and $\hat{\bm{x}}(t)$ be two solutions of the ODE~\eqref{appd: eq_ODE} with different initial conditions, i.e., $\bar{\bm{x}}(0)\neq \hat{\bm{x}}(0)$. then for all $t\in (0, 1]$, $\bar{\bm{x}}(t)\neq \hat{\bm{x}}(t)$. Informally, it states that ODE trajectories cannot intersect.
      \end{theorem}
      
      \begin{figure*}[htb]
        \begin{center}
          \centering
          \subfigure{\label{fig_Infinite_a}}
          \subfigure{\label{fig_Infinite_b}}
          \subfigure{\label{fig_Infinite_c}}
          \includegraphics[width=0.58\textwidth]{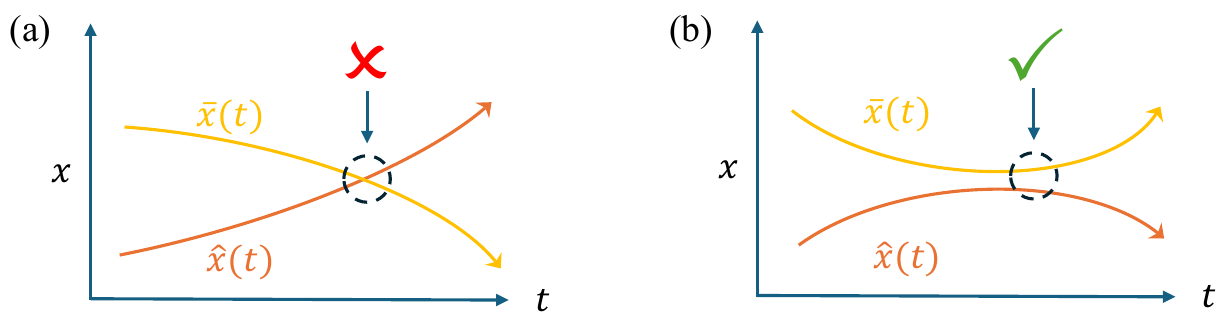}
          \caption{Illustration of the non-intersecting trajectories of ODEs. (a) The two trajectories intersect each other, which is not feasible for ODEs. (b) Any two trajectories cannot intersect each other at any time $t$. }
          \label{fig_Infinite}
        \end{center}
      \end{figure*}
      
      \begin{theorem}[Gronwall's inequality, \cite{howard1998gronwall, dupont2019augmented}]
        Let $\bm{u}_t(\bm{x}): [0, 1]\times \mathbb{R}^d \rightarrow \mathbb{R}^d$ be a continuous function and let $\bar{\bm{x}}(t)$ and $\hat{\bm{x}}(t)$ be two solutions of the ODE~\eqref{appd: eq_ODE}, satisfying the IVP:
        \begin{equation}
          \begin{aligned}
            \frac{{\rm d} \bar{\bm{x}}(t)}{{\rm d} t} &= \bm{u}_t[\bar{\bm{x}}(t)], \quad t\in[0, 1], \quad \bar{\bm{x}}(0) =\bar{\bm{x}}_0, \\      \frac{{\rm d} \hat{\bm{x}}(t)}{{\rm d} t} & = \bm{u}_t[\hat{\bm{x}}(t)], \quad t\in[0, 1], \quad \hat{\bm{x}}(0) =\hat{\bm{x}}_0,
          \end{aligned}
        \end{equation}
        Assume there is a constant $L\geq 0$ such that
        \begin{equation}
          \|\bm{u}_t[\bar{\bm{x}}(t)] - \bm{u}_t[\hat{\bm{x}}(t)]\| \leq L \|\bar{\bm{x}}(t) - \hat{\bm{x}}(t)\|,
        \end{equation}
        Then for $t\in [0, 1]$, we have
        \begin{equation}
          \label{appd:eq_gronwall}
          \|\bar{\bm{x}}(t) - \hat{\bm{x}}(t)\|\leq e^{Lt}\|\bar{\bm{x}}_0 - \hat{\bm{x}}_0\|.
        \end{equation}
      \end{theorem}
      
      \begin{theorem}[Homeomorphism, \cite{younes2010shapes, dupont2019augmented}]
        \label{appd:thm_hom}
        Consider the flow map $\bm{\phi}_t(\bm{x}_0)$ of the ODE~\eqref{appd: eq_ODE}. Then, for all $t\in[0, 1]$,  $\bm{\phi}_t(\bm{x}_0)$ is a homomorphism, i.e.,
        \begin{enumerate}
          \item $\bm{\phi}_t$ is continuous;
          \item $\bm{\phi}_t$ is a bijection;
          \item $\bm{\phi}_t^{-1}$ is continuous.
        \end{enumerate}
      \end{theorem}
      \begin{remark}
        More precisely, the flow map $\bm{\phi}_t(\bm{x}_0)$ of the ODE~\eqref{appd: eq_ODE} is a diffeomorphism \cite{younes2010shapes}, but we will not use this stronger property in our proofs.
      \end{remark}
      
      \subsection{Properties of optimal transport}
      \begin{theorem}[Nondecreasing map, \cite{santambrogio2015optimal}]
        Given $q_0, q_1 \in \mathcal{P}(\mathbb{R})$, suppose that $q_0$ is atomless\footnote{For every $x\in\mathbb{R}$, the probabilistic measure on this single point $x$ is equal to zero.}. Then, there exists a unique nondecreasing map $T_{\rm{mon}}:\mathbb{R} \rightarrow \mathbb{R}$ such that $(T_{\rm{mon}})_{\#}q_0=q_1$.
      \end{theorem}
      \begin{lemma}[Monotonic property, \cite{santambrogio2015optimal}]
        \label{appd:lem_mon}
        Let $\gamma \in \Pi(q_0, q_1)$ be a transport plan between two measures $q_0, q_1 \in \mathcal{P}(\mathbb{R})$. Suppose that it satisfies the property,
        \begin{equation}\label{appd:eq_mon}
          \begin{aligned}
            (x_0, x_1), (x_0^{\prime}, x_1^{\prime}) \in \mathrm{Spt}(\gamma),\\
            x_0 < x_0^{\prime} \implies  x_1<  x_1^{\prime}.
          \end{aligned}
        \end{equation}
        Then, we have $\gamma=\gamma_{\rm{mon}}$. In particular, there is a unique $\gamma$ satisfying \eqref{appd:eq_mon}. Moreover, if $q_0$ is atomless, then $\gamma=\gamma_{T_{\rm{mon}}}$.
      \end{lemma}
      \begin{theorem}[Optimality of the monotone map, \cite{santambrogio2015optimal}]
        \label{appd:thm_mon}
        Let $h:\mathbb{R}\rightarrow \mathbb{R}^{+}$ be a strictly convex function and $q_0, q_1 \in \mathcal{P}(\mathbb{R})$ be probability measures. Consider the cost $c(x_0, x_1) = h(x_1 - x_0)$ and suppose that the Kantorovich problem~\eqref{eq_KP} has a finite value. Then, it has a unique solution, which is given by $\gamma_{\rm{mon}}$. In the case where $q_0$ is atomless, this optimal plan is induced by the map $T_{\rm{mon}}$.
      \end{theorem}
      \begin{theorem}[Cyclical monotonicity, \cite{villani2009optimal}]
        \label{appd:thm_useful}
        Let $\mathcal{X}, \mathcal{Y}$ be arbitrary sets, and $c:\mathcal{X}\times \mathcal{Y} \rightarrow (-\infty, +\infty]$ be a function. A subset $\Gamma \subseteq \mathcal{X}\times \mathcal{Y}$ is said to be $c$-cyclically monotone if, for any $N\in\mathbb{N}$, and any family $(\bm{x}_1, \bm{y}_1),...,(\bm{x}_N, \bm{y}_N)$ of points in $\Gamma$, holds the inequality
        \begin{equation} \label{appd:eq_Cyclical}
          \sum\limits_{i=1}^{N} c(\bm{x}_i, \bm{y}_i) \leq \sum\limits_{i=1}^{N} c(\bm{x}_i, \bm{y}_{i+1})
        \end{equation}
        (with the convention $y_{N+1} = y_1$). A transport plan is said to be $c$-cyclically monotone if it is concentrated on a $c$-cyclically monotone set.
      \end{theorem}
      \begin{remark}
        Informally, a $c$-cyclically monotone plan cannot be improved via perturbations. Consequently, it follows intuitively that an optimal plan should adhere the $c$-cyclical monotonicity \cite{villani2009optimal}.
      \end{remark}
      \begin{theorem}[\cite{villani2009optimal}]
        \label{appd:thm_useful}
        Let $(\mathcal{X}, \mu)$ and $(\mathcal{Y}, \nu)$ be any two Polish probability spaces, let $T$ be a continuous map $\mathcal{X} \rightarrow \mathcal{Y}$, and let $\pi = (\mathrm{Id}, T)_{\#}\mu$ be the associated transport map. Then, for each $\bm{x} \in Spt(\mu)$, the pair $(\bm{x}, T(\bm{x}))$ belongs to the support of $\pi$.
      \end{theorem}
      
      \section{Proofs}\label{sec:proofs}
      
      \printProofs
      
      \begin{figure}[htb]
        \begin{center}
          \centering
          \includegraphics[width=0.55\textwidth]{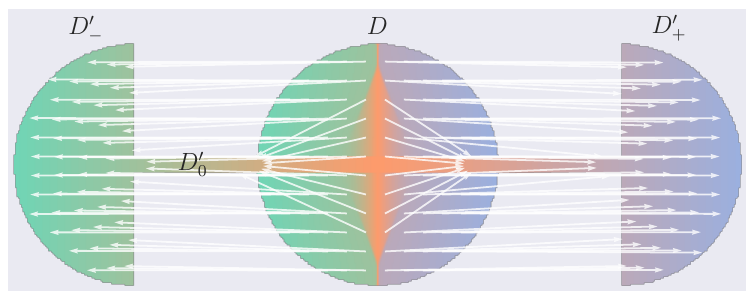}
          \caption{Illustration of the Caffarelli's counterexample.}
          \label{fig_counterexample}
        \end{center}
      \end{figure}
      
      \section{Additional results}
      \label{sec:appd_additional_results}
      
      \textbf{Regularity.} Here, we present a famous result from the study of optimal transport regularity. Specifically, \citet{caffarelli1992regularity} demonstrated that the optimal transport map can be discontinuous when the target measure is supported on a non-convex domain. Therefore, the ideal flow map $\bm{\varphi}_1[\bm{x}(0);\bm{\theta}]$ induced by the optimal transport map is discontinuous as well. This is formalized in the following result.
      \begin{proposition}[Caffarelli's counterexample]\label{thm:fig_counterexample}
        Suppose the source and target distributions $q_0$ and $q_1$ are defined on $2$-d disc $D$ and dumbbell $D^{\prime}$ (see Fig.~\ref{fig_counterexample}, and both normalized to be a probability measure), respectively. Consider the (dynamic) optimal transport problem as defined in Eq.~\eqref{eq:OT} (or Eq.~\eqref{eq:dyn_OT}). If the NODE~\eqref{eq_NODE} exactly solves the problem, then the flow map $\bm{\varphi}_1[\bm{x}(0);\bm{\theta}]$ is discontinuous, as shown in Fig.~\ref{fig_counterexample}.
      \end{proposition}
      \begin{proof}
        For clarity and completeness, we mainly provide the idea of its proof, and more detailed proof can be found in \cite{villani2009optimal}.
        
        Here, the cost function $c(x_0, x_1)$ is the Euclidean distance. Consider the upper regions of the ball $D$, the left half-ball $D_{-}^{\prime}$, and the right half-ball $D_{+}^{\prime}$, respectively. Then, a large fraction (say 0.99) of the mass in $D$ has to go to $D_{-}^{\prime}$ (if it lies on the left) or to $D_{+}^{\prime}$ (if it lies on the right). Due to the homomorphism of the flow map, it should preserve the topology of the source support, i.e., the connectedness. Therefore, the map should transport the mass of a small subset in $D$ into the tube $D_{0}^{\prime}$ connecting the left half-ball $D_{-}^{\prime}$ and the right half-ball $D_{+}^{\prime}$. Moreover, there is some point $\bm{x}_0 \in D$ such that the corresponding target point $\bm{x}_1$, obtained by the transport map, is close to the left end of the tube $D_{0}^{\prime}$.
        
        In particular, without loss of generality, we assume that $\bm{x}_1 - \bm{x}_0$ has a large downward component. From the convergence in probability, many of the neighbors ${\bm{x}}_0^{\prime}$ of $\bm{x}_0$ have to be transported to, say, $D_{-}^{\prime}$, with nearly horizontal displacements ${\bm{x}}_1^{\prime} - {\bm{x}}_0^{\prime}$. If such an ${\bm{x}}_0^{\prime}$ is picked below $\bm{x}_0$, we shall have,
        \begin{equation}
          \langle  {\bm{x}}_0 - {\bm{x}}_0^{\prime},  {\bm{x}}_1 - {\bm{x}}_1^{\prime} \rangle<0,
        \end{equation}
        or equivalently,
        \begin{equation}
          |{\bm{x}}_0 - {\bm{x}}_1|^2 + |{\bm{x}}_0^{\prime} - {\bm{x}}_1^{\prime}|^2 >
          |{\bm{x}}_0 - {\bm{x}}_1^{\prime}|^2 + |{\bm{x}}_0^{\prime} - {\bm{x}}_1|^2.
        \end{equation}
        If the flow map $\bm{\varphi}_1[\bm{x}(0);\bm{\theta}]$ is continuous, in view of Theorem~\ref{appd:thm_useful} this contradicts the $c$-cyclical monotonicity of the optimal coupling. The conclusion is that when the tube $D_{0}^{\prime}$ is ``thin'' enough, the optimal flow map $\bm{\varphi}_1[\bm{x}(0);\bm{\theta}]$ is discontinuous.
      \end{proof}
      \begin{remark}
        In the above Caffarelli's counterexample, though the target space (i.e., the thin dumbbell) of $q_1$ (extremely smooth, constant) is connected but not convex, the optimal transport map can be discontinuous as well. Significantly, since the (convex) source space and the (non-convex) target space have the same topology, one can naturally employ an ODE to construct a dynamic transport map (not optimal) while preserving the topology over time.
      \end{remark}
      
      \textbf{Joint clustering.} As shown in the Fig. \ref{fig_uniform22}(a), the source and target distributions both have two disconnected supports, while the corresponding optimal transport coupling clearly leads to a joint partition with three joint clusters separated by two singular points. Motivated by this observation, we employ the FM to achieve a joint clustering of the source and target pair data points. Specifically, we first use the well-trained FM model $\bm{v}_t(\bm{x};\bm{\theta})$ to construct the pair $(\bm{x}_0, {\bm{x}}_1)$ by solving the IVP~\eqref{eq_NODE}, where $\bm{x}(0) = \bm{x}_0\sim q_0(\bm{x}_0)$ is the initial state and $\bm{x}(1) = {\bm{x}}_1 \sim p_1({\bm{x}}_1) \approx q_1({\bm{x}}_1)$ is the numerically sampled final state. Then, one can use the classical clustering algorithms to partition the constructed pair datasets into $K$ sets, i.e., $\bm{X}^{(1)},...,\bm{X}^{(K)}$. Moreover, it partitions the both supports of the source and target distributions into  $K$ sets as well, i.e., $K_0=K_1=K$, and we, without loss of generality, assume that the set $\bm{X}_0^{(i)}$ and the set $\bm{X}_1^{(i)}$ are paired with equal mass $\rho_0^{(i)} = \rho_1^{(i)}=\rho^{(i)}$. Therefore, the coupling matrix $P$ (as defined above) is a diagonal matrix of the explicit form $P=\text{diag} \{\rho^{(1)}, ..., \rho^{(K)}\}$. Here, we consider the ODE couplings induced by I-CFM and OT-CFM, and we call the corresponding SFM  the IC-SFM and the OTC-SFM, respectively. Notably, under the ODE couplings,  it does not require searching the optimal transport couplings within a data bach for the IC-SFM and the OTC-SFM.
      
      As shown in Fig.~\ref{fig_2gaussians_joint_clustering}, different from the transportation way illustrated in Fig.~\ref{fig_2gaussians}, both the IC-SFM and OTC-SFM can address the singularity problem as well and transport each source data point to the target space in a determined way. However, the identification of the singularity in high-dimensional situations, such as image datasets, is a challenging task, which is out of the scope of this work. We therefore leave it as one of our future directions.
      
      \begin{figure}[htb]
        \begin{center}
          \includegraphics[width=0.55\textwidth]{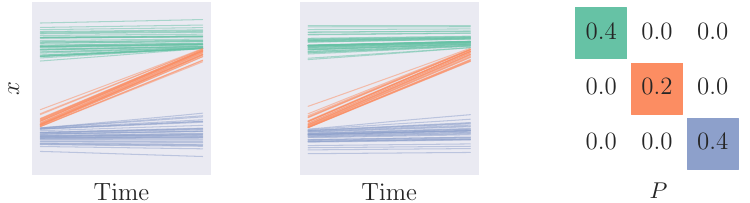}
          \caption{Trajectories of the IC-SFM (left) and the OTC-SFM (middle) with $3$ joint clusters, and the coupling matrix $P$ (right).}
          \label{fig_2gaussians_joint_clustering}
        \end{center}
      \end{figure}
      
      \textbf{Gaussian flow.} In the original work of CFM \cite{tong2023conditional,tong2023improving}, they introduce a more general conditional probability path, defined by:
      \begin{equation}\label{eq:CFM_sigma}
        \begin{aligned}
          p_t(\bm{x}|\bm{z}) &= \mathcal{N}[\bm{x}|(1-t)\bm{x}_0 + t\bm{x}_1, \sigma^2],\\
          \bm{u}_t(\bm{x}|\bm{z}) &= \bm{x}_1 - \bm{x}_0,
        \end{aligned}
      \end{equation}
      also referred as to the Gaussian flow, where the pair $\bm{z} := (\bm{x}_0, \bm{x}_1)$ is sampled from the Independent coupling or optimal transport coupling, and $\sigma$ is typically chosen as a small value, serving as a regularization for optimization. Notably, the CFM mentioned in the main text can be regarded as a special case of Eq.~\eqref{eq:CFM_sigma}, i.e., $\sigma=0$ or equivalently $p_t(\bm{x}|\bm{z}) =\delta_{(1-t)\bm{x}_0 + t\bm{x}_1}$ (the Dirac mass). Naturally, we can generalize our SCFM framework in a similar way, yielding a more general conditional probability path for SCFM:
      \begin{equation}\label{eq:CFM_sigma}
        \begin{aligned}
          p_t(\bm{x}|\bm{z}, \bm{s}) &= \mathcal{N}[\bm{x}|(1-t)\bm{x}_0 + t\bm{x}_1, \sigma^2],\\
          \bm{u}_t(\bm{x}|\bm{z}, \bm{s}) &= \bm{x}_1 - \bm{x}_0,
        \end{aligned}
      \end{equation}
      where the pair $\bm{z} := (\bm{x}_0, \bm{x}_1)$ is now sampled from the Independent coupling or optimal transport coupling condition on the switching signal $\bm{s}$, i.e., $\bm{z}\sim q(\bm{z}|\bm{s})$. Notably, we set $\sigma=0$ for all experiments and leave the impact of different values of $\sigma$ on the experimental results for future study, which is out of scope in this work.
      
      \textbf{Switching coupling.} For easy implementation of the switching coupling, we make use of the batch data to construct an implicit switching signal distribution. Specifically, given batch data, we employ the following construction:
      \begin{enumerate}
        \item Sample $\left\{ \bm{x}_0^{(k)} \right\}_{k=1}^m \sim q_0(\bm{x}_0)$ and $\left\{ \bm{x}_1^{(k)} \right\}_{k=1}^m \sim q_1(\bm{x}_1)$ with the corresponding label sets $\left\{ {y}_0^{(k)} \right\}_{k=1}^m$ and $\left\{ {y}_1^{(k)} \right\}_{k=1}^m$, respectively.
        \item Construct the batch coupling matrix $P^{m} \in \mathbb{R}^{m\times m}$ using the batch label data as follows:
        \begin{equation}\label{eq:Pm}
          P^{m}(i, j) = \frac{P\left({y}_0^{(i)}, {y}_1^{(j)}\right)}{\text{Count}\left({y}_0^{(i)}, {y}_1^{(j)}\right)},    \end{equation}
          where $P\in\mathbb{R}^{K_0\times K_1}$ is a predefined coupling matrix (see Eq.~\eqref{eq:coupling_matrix}), and $\text{Count}\left({y}_0^{(i)}, {y}_1^{(j)}\right)$ represents the number of the pair $\left({y}_0^{(i)}, {y}_1^{(j)}\right)$ within the batch data,
          \item Sample the switching signal $\bm{s}:=\left({y}_0^{(i)}, {y}_1^{(j)}\right)$ based on the coupling matrix $P^{m}$, and obtain the corresponding data pair $\bm{z}^{(i, j)}:=\left(\bm{x}_0^{(i)}, \bm{x}_1^{(j)}\right)$, resulting in $n$ samples, defined by:
          \begin{equation}
            \left\{ \left(\hat{\bm{x}}_0^{(k)}, \hat{\bm{x}}_1^{(k)}, \hat{y}_0^{(k)}, \hat{y}_1^{(k)} \right)\right\}_{k=1}^n,
          \end{equation}
          where $n$ need not match the batch size $m$, but for simplicity in our experiments we choose $n=m$,
          \item (Optional for OT-SFM) Construct the joint distribution $\pi_{\text{batch}}(\bm{z}|\bm{s})$ induced by the optimal coupling in terms of the Euclidean distance between data points of the pair $\hat{\bm{z}}^{(i, j)}:=\left(\hat{\bm{x}}_0^{(i)}, \hat{\bm{x}}_1^{(j)}\right)$  within the same switching signal data pair set $\left\{(\bm{z}, \bm{s})|\bm{s}=\left(\hat{{y}}_0^{(i)}, \hat{{y}}_1^{(j)}\right)\right\}$,
          \item (Optional for OT-SFM) Sample from the joint distribution $\pi_{\text{batch}}(\bm{z}|\bm{s})$, yielding the samples
          \begin{equation}
            \left\{ \left(\bar{\bm{x}}_0^{(k)}, \bar{\bm{x}}_1^{(k)}, \bar{y}_0^{(k)}, \bar{y}_1^{(k)} \right)\right\}_{k=1}^l,
          \end{equation}
          where $l$ need not match the batch size $m$ or $n$, but for simplicity in our experiments we choose $l=n=m$.
        \end{enumerate}
        
        \begin{proposition}\label{thm:switching_sampling}
          The joint distribution $q^{m}(\bm{s}):=q^{m}(y_0, y_1)$ induced by the batch coupling matrix $P^{m}$ constructed in the above Steps [1-3] has the same distribution $q^{\circ}(\bm{s}):=q^{\circ}(y_0, y_1)$ induced by the coupling matrix $P$, i.e., $q^{m}(y_0, y_1) = q^{\circ}(y_0, y_1)=P(y_0, y_1)$.
        \end{proposition}
        \begin{proof}
          For an arbitrary test function $f(y_0, y_1)$, by construction of the batch coupling matrix $P^{m}$ (see Eq.~\eqref{eq:Pm}),  it holds
          \begin{equation}
            \begin{aligned}
              \mathbb{E}_{q^{m}(y_0, y_1)} f(y_0, y_1)
              &=  \sum \limits_{y_0, y_1} q^{m}(y_0, y_1) f(y_0, y_1)\\
              &=  \sum \limits_{y_0, y_1} \sum\limits_{y_0^{(i)}=y_0, y_1^{(j)}=y_1}  \frac{P\left({y}_0^{(i)}, {y}_1^{(j)}\right)}{\text{Count}\left({y}_0^{(i)}, {y}_1^{(j)}\right)} f(y_0, y_1)\\
              & = \sum \limits_{y_0, y_1} P\left({y}_0, {y}_1\right) f(y_0, y_1)\\
              & = \mathbb{E}_{q^{\circ}(y_0, y_1)} f(y_0, y_1).
            \end{aligned}
          \end{equation}
          The proof is complete.
        \end{proof}
        
        \textbf{Algorithms.} For a better understanding of the proposed SFM, we provide the main pseudocode as shown in Algorithm~\ref{alg:sfm} as well as the pseudocode for constructing the switching coupling as shown in Algorithm~\ref{alg:sc} with the optional module for OT-SFM (see Algorithm~\ref{alg:otsfm}). For inference,  the pseudocode is displayed in Algorithm~\ref{alg:inference}. The code is available at \href{https://github.com/zhuqunxi/switched-flow-matching}{\texttt{https://github.com/zhuqunxi/switched-flow-matching}}. %We will release our source code for all experiments once the quality of the code is met.
        
        \definecolor{commentcolor}{RGB}{110,154,155}   \definecolor{inputcolor}{RGB}{255, 105, 180}
        \newcommand{\PyComment}[1]{\ttfamily\textcolor{commentcolor}{\# #1}}  \newcommand{\PyInput}[1]{\ttfamily\textcolor{inputcolor}{\# #1}}
        \newcommand{\PyCode}[1]{\ttfamily\textcolor{black}{#1}}
        \begin{algorithm}[h]
          \caption{\PyCode{Switched Flow Matching}}
          \begin{flushleft}
            \PyInput{Input: Data=$\{\bm{x}_0, \bm{x}_1, y_0, y_1\}$, sampled from $q_0(\bm{x}_0)$ and $q_1(\bm{x}_1)$} \\
            \PyInput{Output: Model $\bm{v}_t(\bm{x};\bm{\theta}|\bm{s})$ for the vector field given the switching signal $\bm{s}=(y_0, y_1)$} \\
            \PyCode{Initialize Model} \PyComment{Model can be arbitrary neural network structure}\\
            \PyCode{Initialize OT} \PyComment{OT=True/False indicates whether OT-SFM is adopted}\\
            \PyCode{Initialize $P$} \PyComment{$P$ is the predefined coupling matrix}\\
            \PyCode{for $\bm{x}_0, \bm{x}_1, y_0, y_1$ in Data:}
            \PyComment{Generate batch data} \\  \PyCode{~~~~Optimizer.zero\_grad()} \\
            \PyCode{~~~~{\color{blue}$\bm{x}_0, \bm{x}_1, y_0, y_1$ = Sample\_plan($\bm{x}_0$, $\bm{x}_1$, $y_0$, $y_1$, $P$, OT)}}
            \PyComment{Construct switching coupling} \\
            \PyCode{~~~~$\bm{s}$ = $(y_0, y_1)$} \PyComment{Switching signal} \\
            \PyCode{~~~~t = torch.rand(batchsize)}~~\PyComment{Randomly sample  $t \in$ [0,1]} \\
            \PyCode{~~~~\color{blue}{Loss = \{ Model[(1-t)*$\bm{x}_0$ + t*$\bm{x}_1$,~t,~$\bm{s}$] - ($\bm{x}_1$ - $\bm{x}_0$) \}.pow(2).mean()}} \\
            \PyCode{~~~~Loss.backward()} \\
            \PyCode{~~~~Optimizer.step()} \\
            \PyCode{return Model}
          \end{flushleft}
          \label{alg:sfm}
        \end{algorithm}
        
        \begin{algorithm}[h]
          \caption{\PyCode{Switching Coupling}}
          \begin{flushleft}
            \PyCode{def {\color{blue} Sample\_plan($\bm{x}_0$, $\bm{x}_1$, $y_0$, $y_1$, $P$, OT)}}
            \PyComment{Construct switching coupling} \\
            ~~~~\PyComment{$\bm{x}_0$, $\bm{x}_1$: shape=(m, dim); ${y}_0$, ${y}_1$: shape=(m, )}\\
            \PyCode{~~~~m, $K_0$, $K_1$ = len($y_0$)}, $P$.shape[0], $P$.shape[1] \\
            \PyCode{~~~~$P^m(i, j)=P(y_0[i], y_1[j])/\text{Count}(y_0[i], y_1[j])$}~~\PyComment{Construct the batch coupling matrix} \\
            ~~~~\PyComment{Sample from batch coupling matrix (the following three lines)}\\
            \PyCode{~~~~choices = np.random.choice(m*m, p=$P^m$.flatten(), size=m)} \\
            \PyCode{~~~~index0, index1 = np.divmod(choices, m)} \\
            \PyCode{~~~~$\bm{x}_0$, $\bm{x}_1$, $y_0$, $y_1$ = $\bm{x}_0$[index0], $\bm{x}_1$[index1], $y_0$[index0], $y_1$[index1]} \\
            \PyCode{~~~~if not OT:}~~\PyComment{Return the sampled data for I-SFM}\\
            \PyCode{~~~~~~~~return $\bm{x}_0$, $\bm{x}_1$, $y_0$, $y_1$}\\
            ~~~~\PyComment{Sample from $\pi_{\text{batch}}$ induced by OT-SFM (the following three lines)}\\
            \PyCode{~~~~Y\_pair = $y_0$ * $K_1$ + $y_1$ }~~\PyComment{Construct the switching signal}\\
            \PyCode{~~~~M\_mask = Y\_pair[:, None] == Y\_pair[None, :] }~~\PyComment{Construct the mask matrix for sampling the data within the same switching signal data pair set}\\
            \PyCode{~~~~\color{blue}index0, index1 = Sample\_plan\_with\_mask($\bm{x}_0$, $\bm{x}_1$,  M\_mask) }~~\PyComment{Sample from $\pi_{\text{batch}}$}\\
            \PyCode{return $\bm{x}_0$[index0], $\bm{x}_1$[index1], $y_0$[index0], $y_1$[index1]}
          \end{flushleft}
          \label{alg:sc}
        \end{algorithm}
        
        \begin{algorithm}[h]
          \caption{\PyCode{Sample from $\pi_{\text{batch}}$}}
          \begin{flushleft}
            \PyCode{def {\color{blue} Sample\_plan\_with\_mask($\bm{x}_0$, $\bm{x}_1$, M\_mask)}}
            \PyComment{Sample from $\pi_{\text{batch}}$} \\
            ~~~~\PyComment{$\bm{x}_0$, $\bm{x}_1$: shape=(m, dim);  M\_mask: shape=(m, m)}\\
            ~~~~\PyComment{The following lines are adapted from the source code \url{https://github.com/atong01/conditional-flow-matching/blob/main/torchcfm/optimal_transport.py}}\\
            ~~~~\PyCode{m = len($y_0$)}\\
            ~~~~\PyCode{a, b = pot.unif(m), pot.unif(m)}~~\PyComment{Uniform weights for each sample}\\
            ~~~~\PyCode{D = torch.cdist(x0, x1) ** 2}~~\PyComment{Distance matrix}\\
            ~~~~\PyCode{D = D + (1 - M\_mask) * 1e10}~~\PyComment{Refined distance matrix by assigning large values for miss match pairs}\\
            ~~~~\PyCode{p = pot.emd(a, b, M.detach().cpu().numpy())}~~\PyComment{Return the OT matrix} \\
            ~~~~\PyCode{choices = np.random.choice(m*m, p=p, size=m)}~~\PyComment{Sample from the OT matrix}\\
            ~~~~\PyCode{return np.divmod(choices, m)} \\
          \end{flushleft}
          \label{alg:otsfm}
        \end{algorithm}
        
        \begin{algorithm}[h]
          \caption{\PyCode{Inference}}
          \begin{flushleft}
            \PyInput{Input: Model, source Data=$\{\bm{x}_0, y_0\}$}, and coupling matrix $P$ \\
            \PyInput{Output: Generated samples} \\
            \PyCode{Samples = []}\\
            \PyCode{for $\bm{x}_0, y_0$ in Data:}\\
            ~~~~\PyCode{$y_1$ = sampler($y_0$, P)} ~~\PyComment{Sample $y_1$ based on $P$ and $y_0$} \\
            ~~~~\PyCode{$\bm{s}$ = ($y_0$, $y_1$)}~~\PyComment{Sampled switching signal}\\
            ~~~~\PyCode{$\bm{x}_1$ = model.ODE\_solver($\bm{x}_0$, $\bm{s}$)}\\
            ~~~~\PyCode{Samples.append($\bm{x}_1$)} ~~\PyComment{Generate a target sample via the ODE solver}\\
            \PyCode{return Samples}\\
          \end{flushleft}
          \label{alg:inference}
        \end{algorithm}
        
        \textbf{Additional experimental results.} The additional experimental results are summarized as follows.
        \begin{itemize}
          \item \textbf{Figure~\ref{fig_2uniform_shrink}}: An experimental illustration of the limitations as pointed out in the bulletins~\ref{thm1_it1}~\&~\ref{thm1_it2} of Proposition~\ref{thm:uniform21} as well as in the Corollary~\ref{thm:corol}. Specifically, when the source data is near the singularity point, the corresponding flow is ``out of distribution'', i.e., transporting the source data to the target space but out of the target distribution (see the second row of the Fig.~\ref{fig_2uniform_shrink}). Particularly, when the source and target distributions degenerate into the Dirac masses (see the right column of the Fig.~\ref{fig_2uniform_shrink}), the CFM cannot solve the transport problem, since the flow map of the CFM is determined due to the existence and uniqueness theorem. On the contrary, our proposed SFM can address these limitations.
          \item \textbf{Figure~\ref{fig_2uniform_shift}}: An experimental illustration of the limitations as pointed out in the bulletins~\ref{thm1_it3} of Proposition~\ref{thm:uniform21}. As the gap between the two modes of target distribution increases, data near the singularity will be mapped to far data points. In other words, the flow map's output (or value) can vary rapidly with small changes in the source singular point, signifying a large Lipchitz constant of the flow map or a worse generalization. In addition, this can present challenges in terms of the numerical stability of the ODE solver. On the contrary, our proposed SFM does not have these issues.
          \item \textbf{Table~\ref{tab:infinite}}: An experimental illustration of the example of an infinite number of singular points in the Propositioin~\ref{thm:infinite}. The learned flow maps of the CFM, including the I-CFM and the OT-CFM, cannot solve the transportation problem due to the singularity, while our proposed SFM performs very well.
          \item \textbf{Table~\ref{tab:gaussian2checkerboard_nfe}}: Generated target samples from the trained CFM and the SFM at $t$ = 1 for different NFE on transporting the $2$-d Gaussian mixture (8 modes) to the checkerboard. It is observed that both the I-CFM and the OT-CFM face significant challenges in effectively learning the transportation flows. The SFM works well when using the adaptive step size ODE solver, but the I-SFM performs much worse than the OT-SFM for small NFEs due to the straightness issue of the learned flows.
          \item \textbf{Table~\ref{tab:gaussian2checkerboard_itr}}: Generated target samples from the trained CFM and the SFM at $t$ = 1 for different training iterations on transporting the $2$-d Gaussian mixture (8 modes) to the checkerboard. Notably, the SFM demonstrates improved efficiency in training and faster convergence compared to the CFM, owing to its superior regularity.
          \item \textbf{Figure~\ref{fig_true_samples}}: True samples from the source distribution Gaussian mixture and the target distribution, i.e., the CIFAR-10 image dataset. Different from the existing works \cite{lipman2022flow, liu2022flow, pooladian2023multisample, tong2023conditional, tong2023improving}, to illustrate the efficiency of our proposed SFM under the heterogeneity in both $q_0$ and $q_1$, we here consider the Gaussian mixture with $2$ modes as the source distribution instead of the standard Gaussian distribution.
          \item \textbf{Table~\ref{tab:cifar10_samples_I} (resp., Table~\ref{tab:cifar10_samples_OT})}: Generated target samples from the trained I-CFM (resp., OT-CFM) and the I-SFM (resp., OT-SFM) at $t$ = 1 for different NFE on transporting the Gaussian mixture ($2$ modes) to the CIFAR-10 image dataset.
        \end{itemize}
        
        \newpage
        \begin{figure*}[htb]
          \begin{center}
            \centerline{\includegraphics[width=0.98\textwidth]{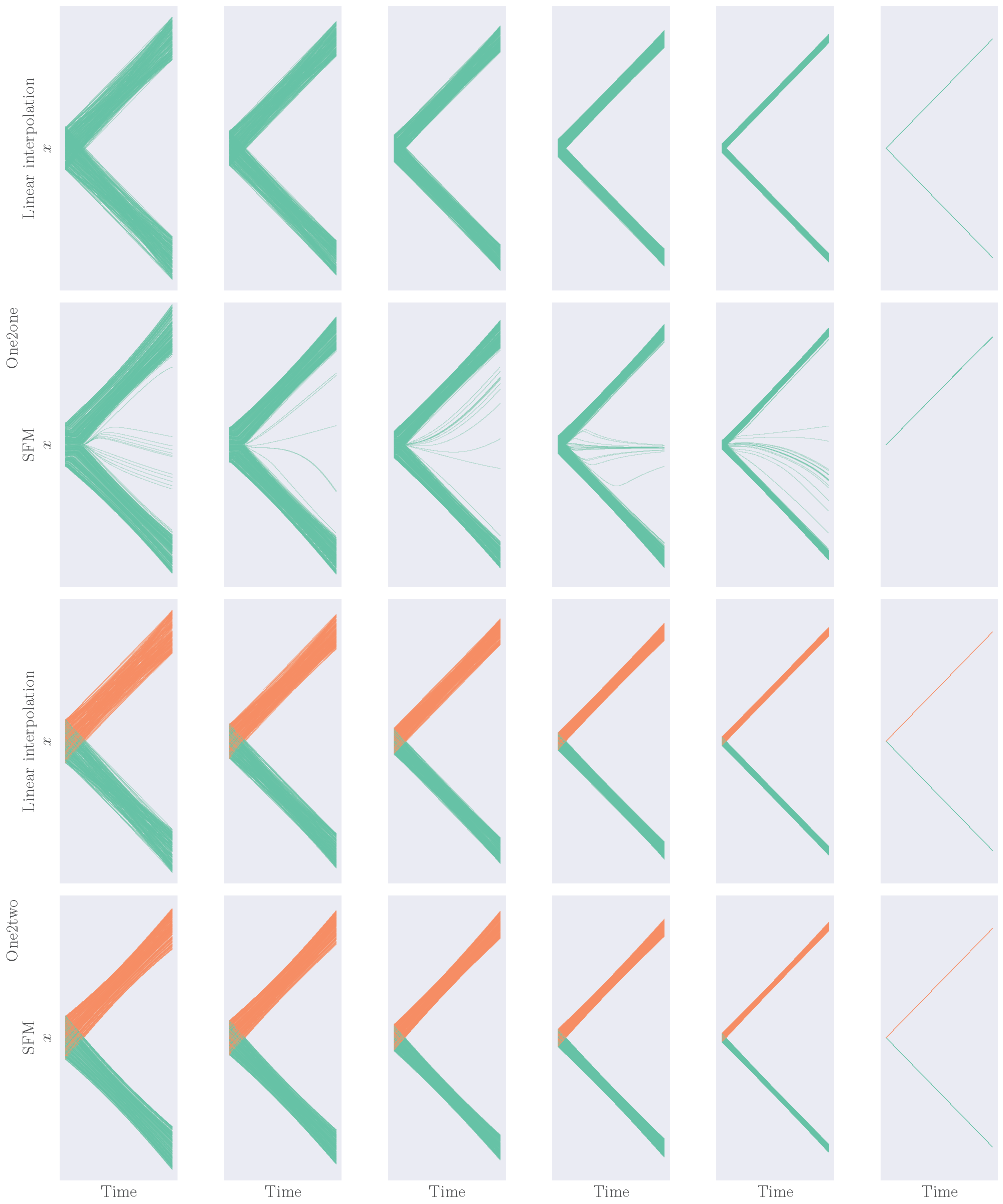}}
            % \vskip -0.1in
            \caption{Linear interpolation of the source and target data (sampled from the unimodal and bimodal uniform distributions), and the trajectories trained from two kinds of SFM, i.e., one2one ($1$ switching signal, equivalently CFM) and one2two ($2$ switching signals), via shrinking the size of the supports (from the left column to the right column). In the right column, the supports of the source and target distributions are both shrunk to a Dirac measure and a $2$-mixture Dirac measure, respectively. }
            \label{fig_2uniform_shrink}
          \end{center}
          % \vskip -0.2in
        \end{figure*}
        
        \newpage
        \begin{figure*}[htb]
          \begin{center}
            \centerline{\includegraphics[width=0.98\textwidth]{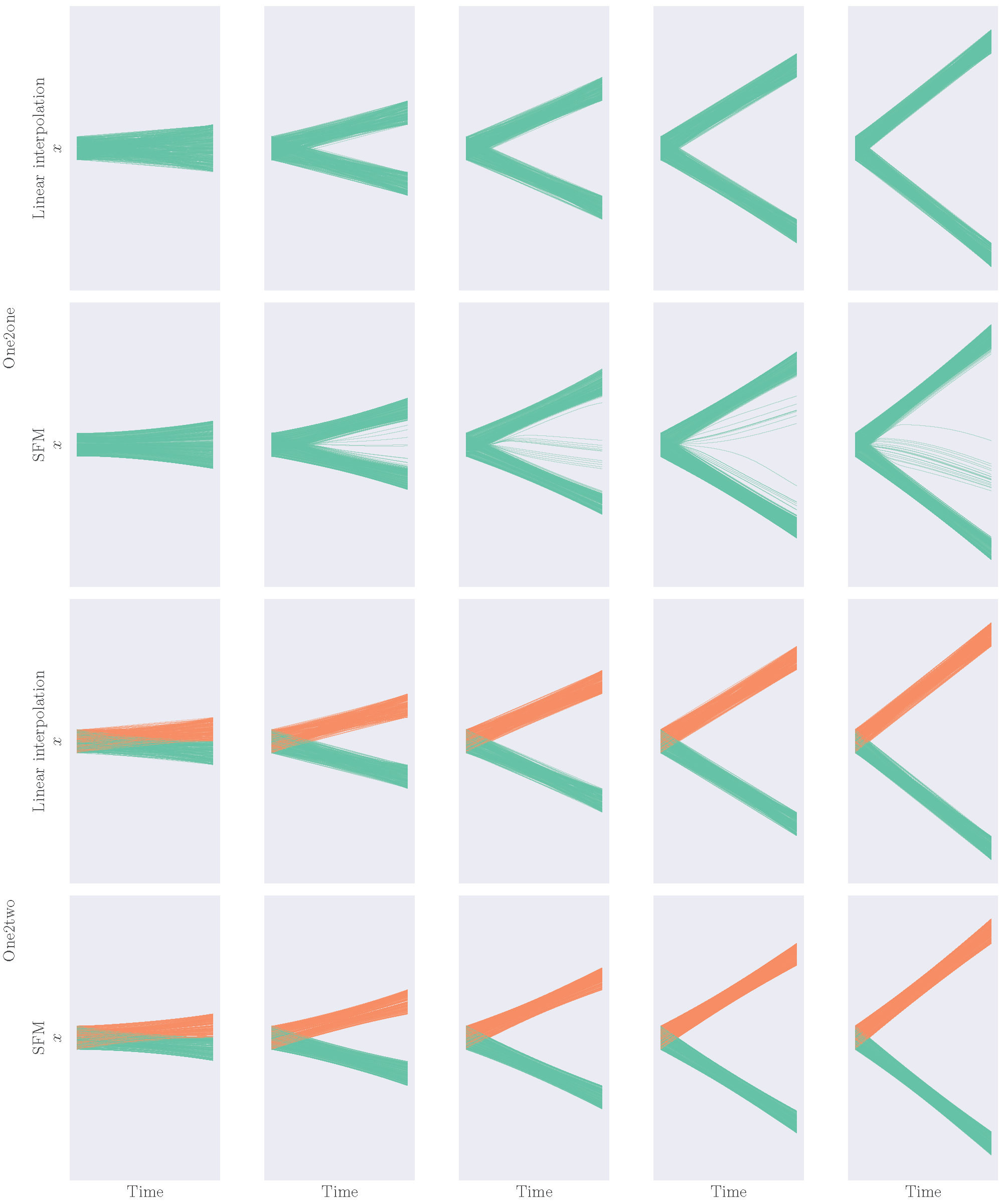}}
            % \vskip -0.1in
            \caption{Linear interpolation of the source and target data (sampled from the unimodal and bimodal uniform distributions), and the trajectories trained from two kinds of SFM, i.e., one2one ($1$ switching signal, equivalently CFM) and one2two ($2$ switching signals), via enlarging the distance of the two modes of the target distribution (from the left column to the right column). }
            \label{fig_2uniform_shift}
          \end{center}
          % \vskip -0.2in
        \end{figure*}
        
        \newpage
        \begin{table}[htb]
          \caption{Trajectories of the learned CFM and SFM via varying the number of function evaluations NFE on the example of infinite number of singular points in the Propositioin~\ref{thm:infinite}.}
          \label{tab:infinite}
          \vspace*{1em}
          \resizebox{\linewidth}{!}{        \centering
          \begin{tabular}{m{1.5cm}<{\centering}|m{5cm}<{\centering}m{5cm}<{\centering}m{5cm}<{\centering}m{5cm}<{\centering}m{5cm}<{\centering}}
            \toprule
            \midrule
            Method  & NFE=1 & NFE=2 & NFE=5 &  NFE=Adaptive\\
            \midrule
            I-CFM $~~$or$~~$ I-SFM (one2one) & \includegraphics[width=0.3\textwidth]{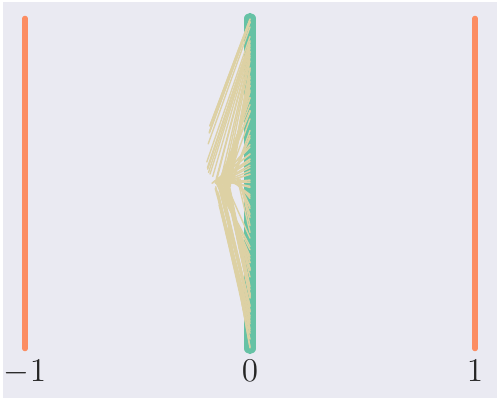} & \includegraphics[width=0.3\textwidth]{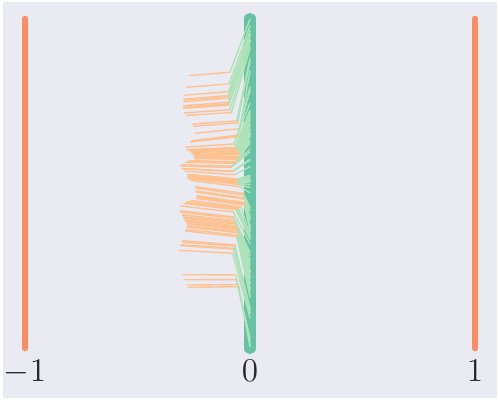} & \includegraphics[width=0.3\textwidth]{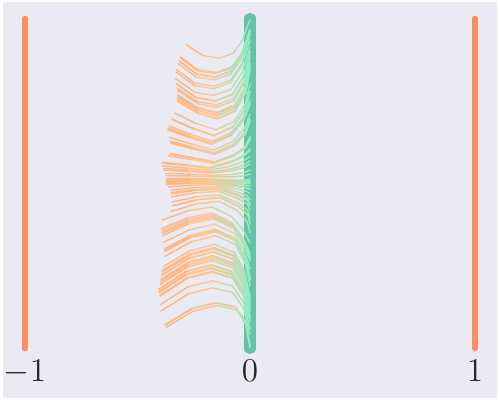} & \includegraphics[width=0.3\textwidth]{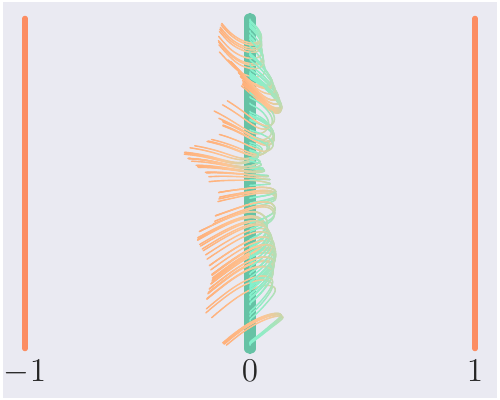}\\
            I-SFM (one2two) & \includegraphics[width=0.3\textwidth]{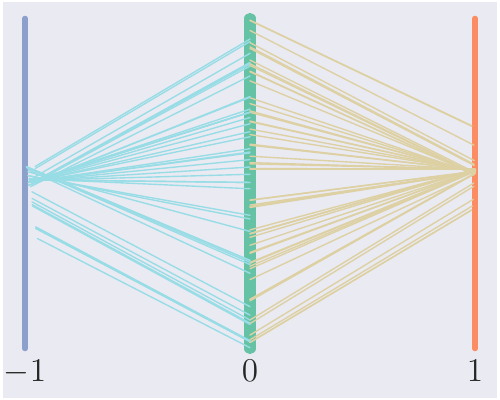} & \includegraphics[width=0.3\textwidth]{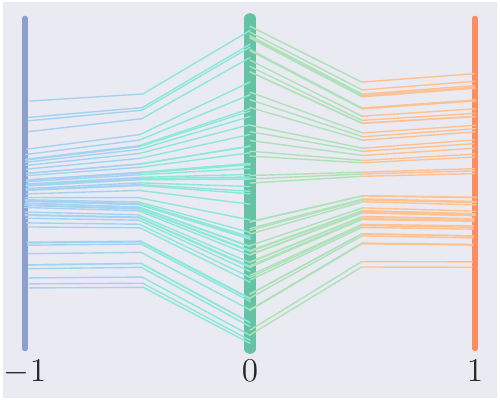} & \includegraphics[width=0.3\textwidth]{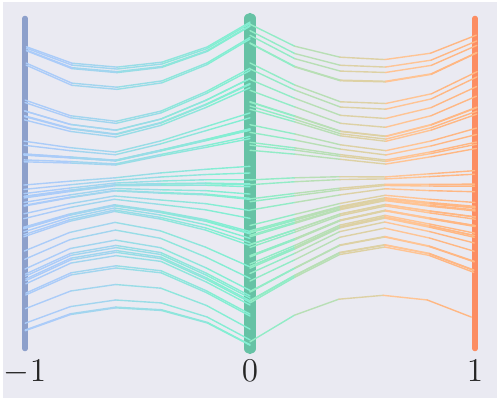} & \includegraphics[width=0.3\textwidth]{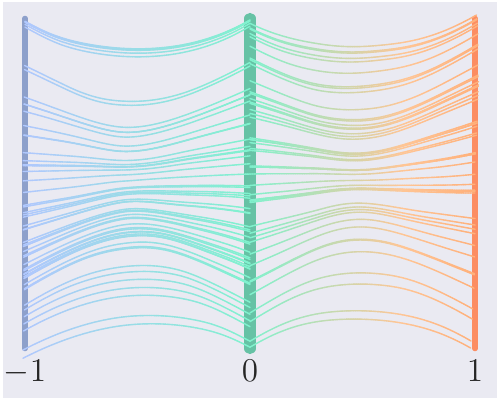}\\
            I-SFM (two2two) & \includegraphics[width=0.3\textwidth]{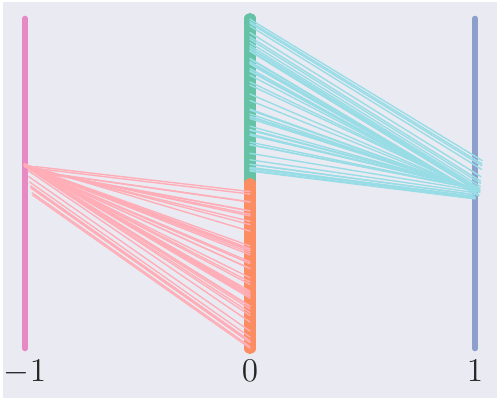} & \includegraphics[width=0.3\textwidth]{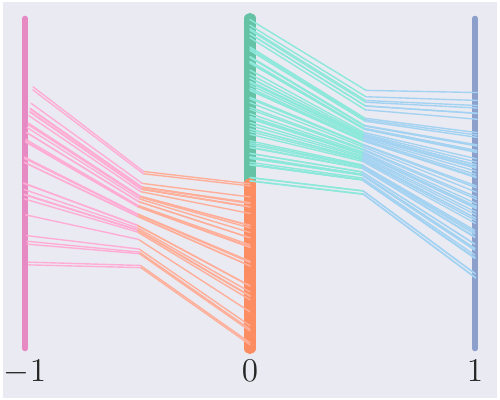} & \includegraphics[width=0.3\textwidth]{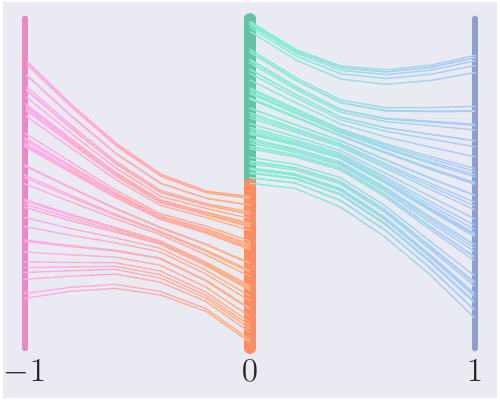} & \includegraphics[width=0.3\textwidth]{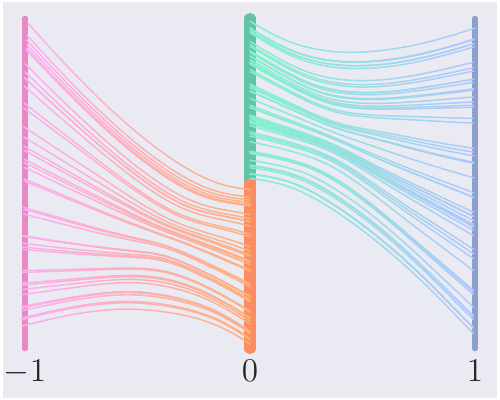}\\
            OT-CFM or OT-SFM (one2one) & \includegraphics[width=0.3\textwidth]{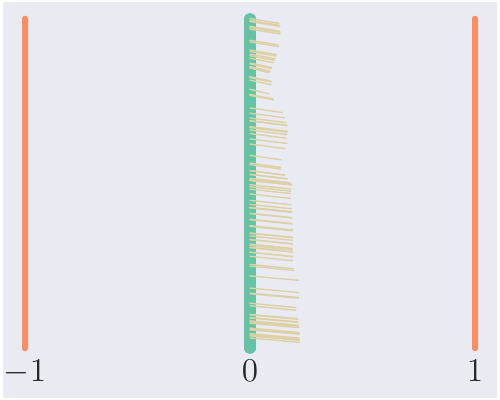} & \includegraphics[width=0.3\textwidth]{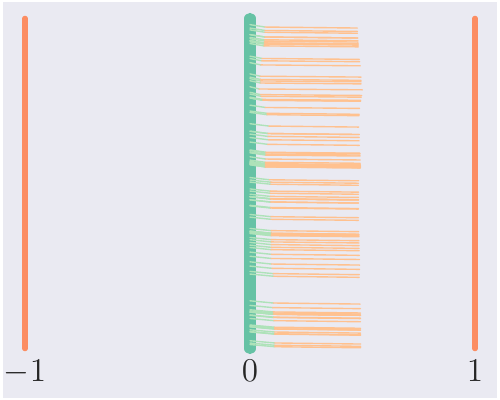} & \includegraphics[width=0.3\textwidth]{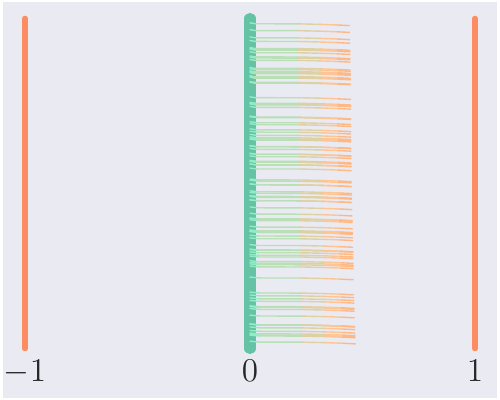} & \includegraphics[width=0.3\textwidth]{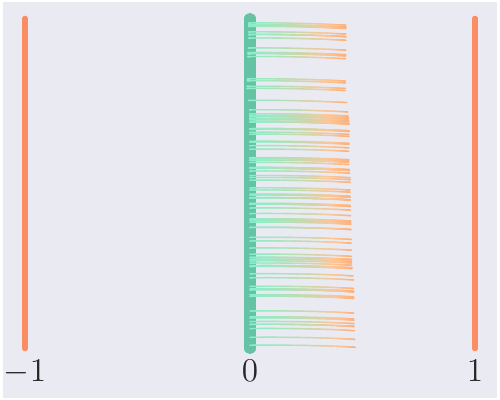}\\
            OT-SFM (one2two) & \includegraphics[width=0.3\textwidth]{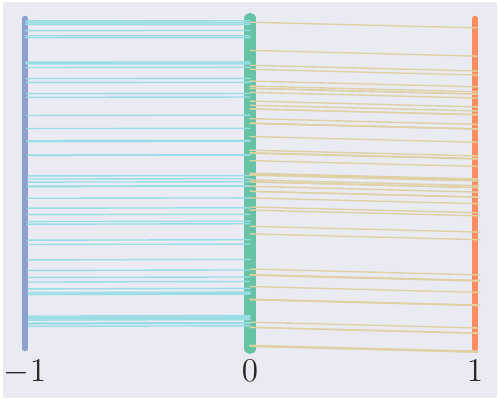} & \includegraphics[width=0.3\textwidth]{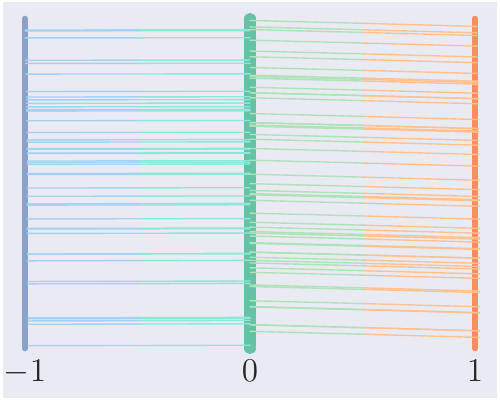} & \includegraphics[width=0.3\textwidth]{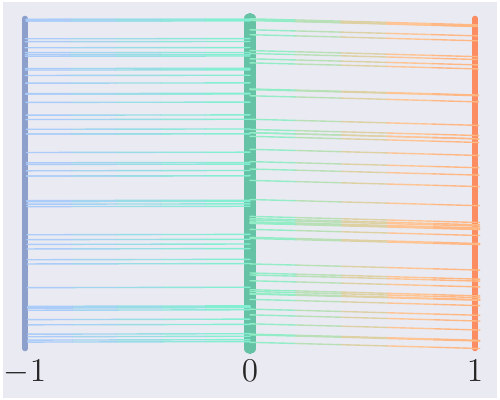} & \includegraphics[width=0.3\textwidth]{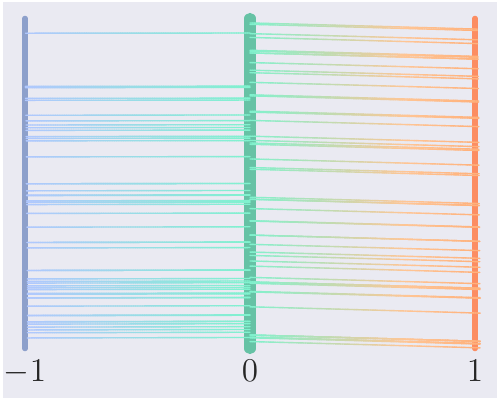}\\
            OT-SFM (two2two) & \includegraphics[width=0.3\textwidth]{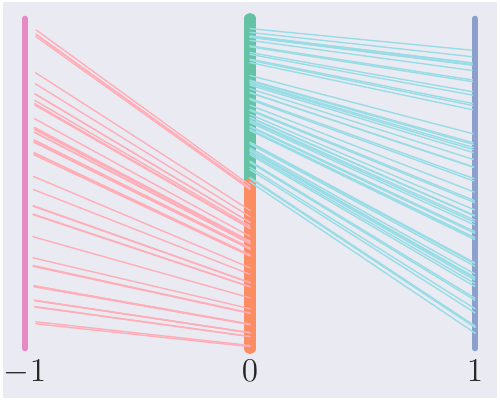} & \includegraphics[width=0.3\textwidth]{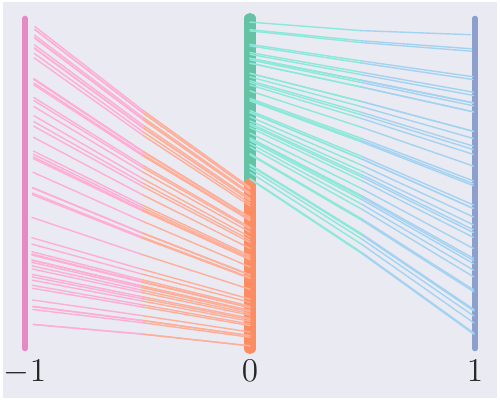} & \includegraphics[width=0.3\textwidth]{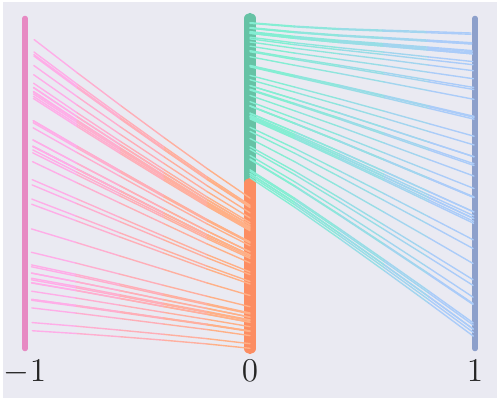} & \includegraphics[width=0.3\textwidth]{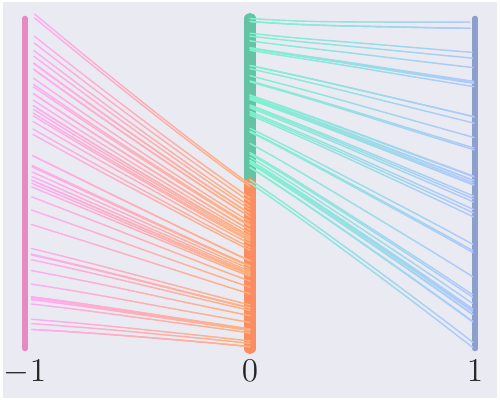}\\
            \bottomrule
            \toprule
          \end{tabular}
          }
        \end{table}
        
        \newpage
        \begin{table}[htb]
          \caption{Generated target samples from the trained CFM and the SFM at $t$ = 1 for different NFE on transporting the $2$-d Gaussian mixture (8 modes) to the checkerboard.}
          \vspace*{1em}
          \label{tab:gaussian2checkerboard_nfe}
          \resizebox{\linewidth}{!}{        \centering
          \begin{tabular}{m{3.5cm}<{\centering}|m{5cm}<{\centering}m{5cm}<{\centering}m{5cm}<{\centering}m{5cm}<{\centering}m{5cm}<{\centering}}
            \toprule
            \midrule
            Method & I-CFM or I-SFM (one2one) & I-SFM (eight2eight) & OT-CFM or OT-SFM (one2one) & OT-SFM (eight2eight)\\
            \midrule
            Source data & \includegraphics[width=0.3\textwidth]{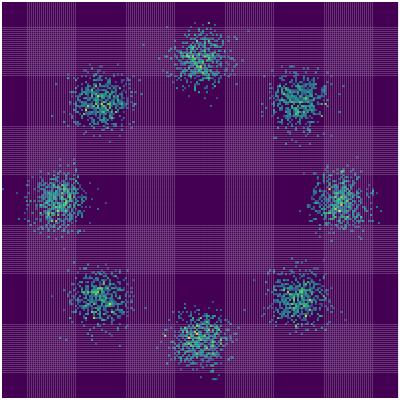} & \includegraphics[width=0.3\textwidth]{Figures_append/gaussian2checkerboard_png/fig_eight2eight_1_source.png} & \includegraphics[width=0.3\textwidth]{Figures_append/gaussian2checkerboard_png/fig_eight2eight_1_source.png} & \includegraphics[width=0.3\textwidth]{Figures_append/gaussian2checkerboard_png/fig_eight2eight_1_source.png}\\
            NFE=1 & \includegraphics[width=0.3\textwidth]{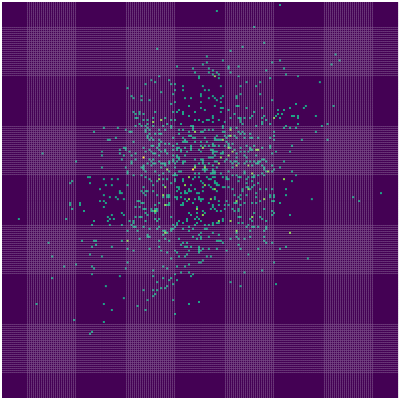} & \includegraphics[width=0.3\textwidth]{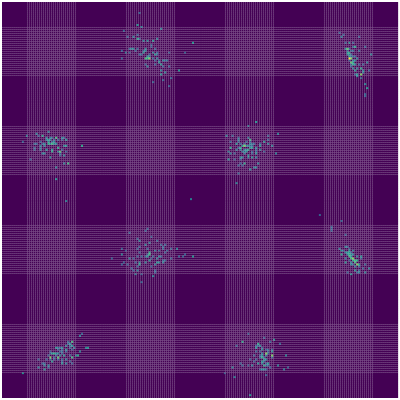} & \includegraphics[width=0.3\textwidth]{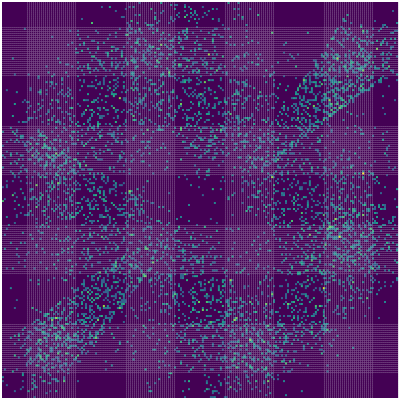} & \includegraphics[width=0.3\textwidth]{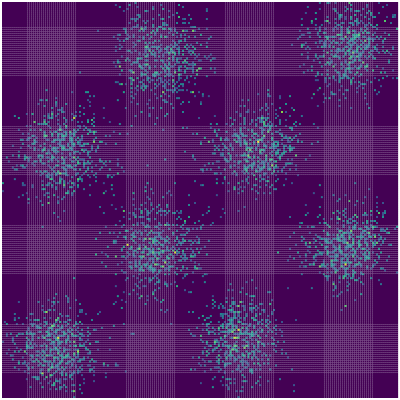}\\
            NFE=2 & \includegraphics[width=0.3\textwidth]{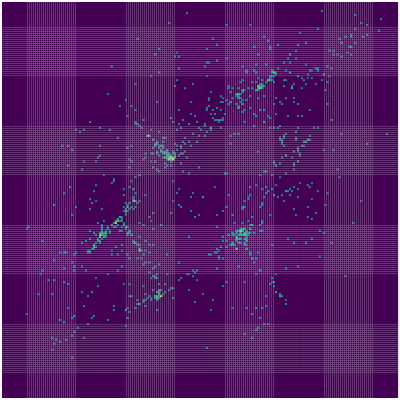} & \includegraphics[width=0.3\textwidth]{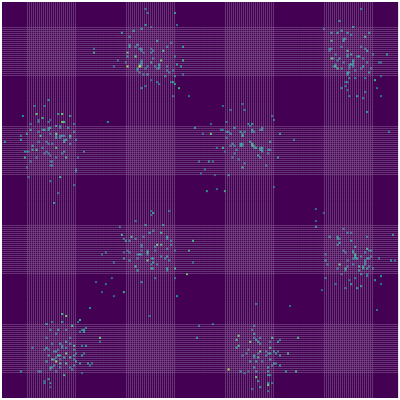} & \includegraphics[width=0.3\textwidth]{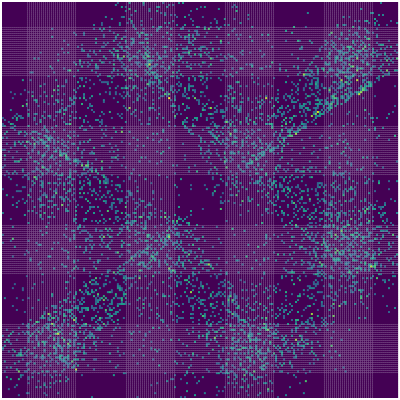} & \includegraphics[width=0.3\textwidth]{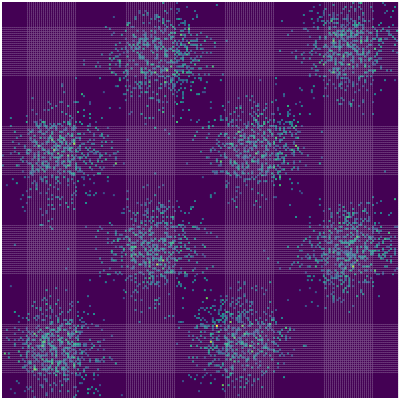}\\
            NFE=5 & \includegraphics[width=0.3\textwidth]{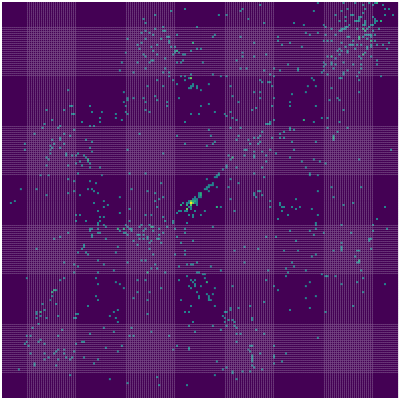} & \includegraphics[width=0.3\textwidth]{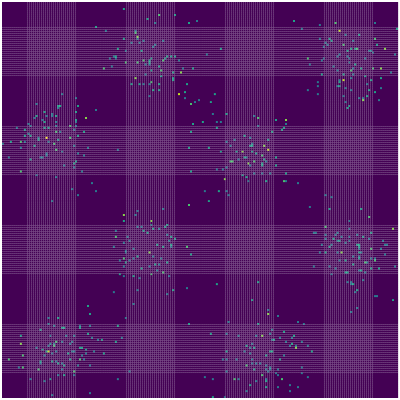} & \includegraphics[width=0.3\textwidth]{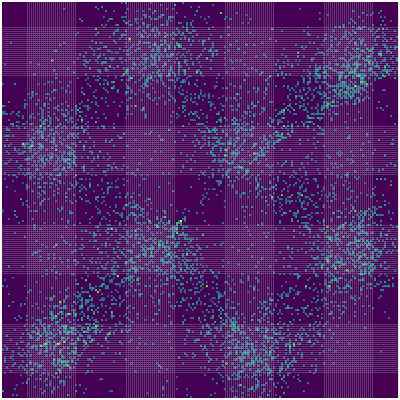} & \includegraphics[width=0.3\textwidth]{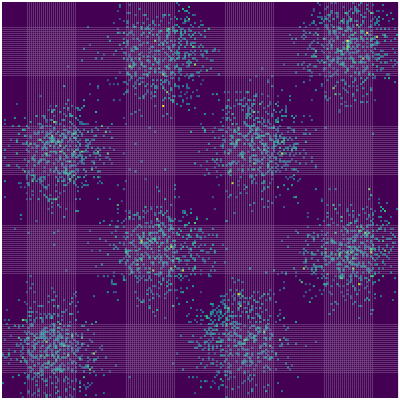}\\
            NFE=Adaptive & \includegraphics[width=0.3\textwidth]{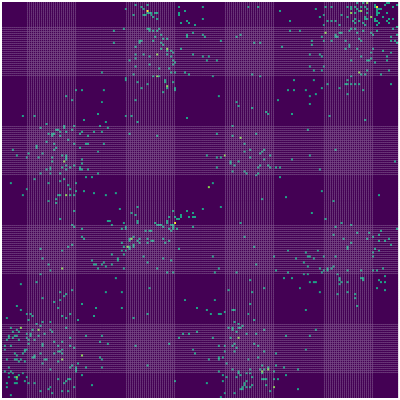} & \includegraphics[width=0.3\textwidth]{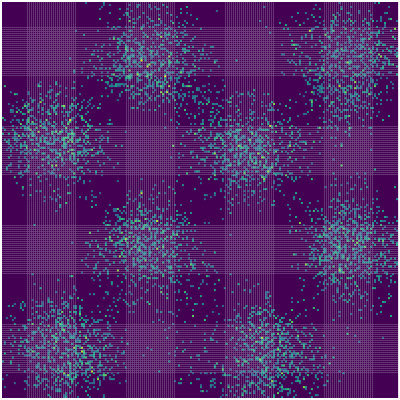} & \includegraphics[width=0.3\textwidth]{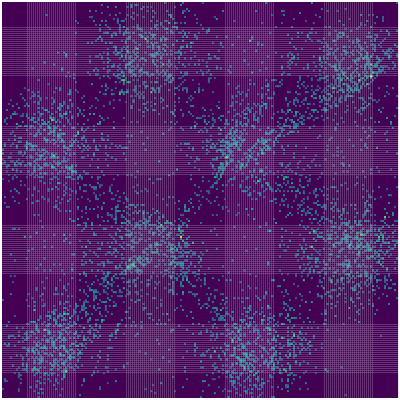} & \includegraphics[width=0.3\textwidth]{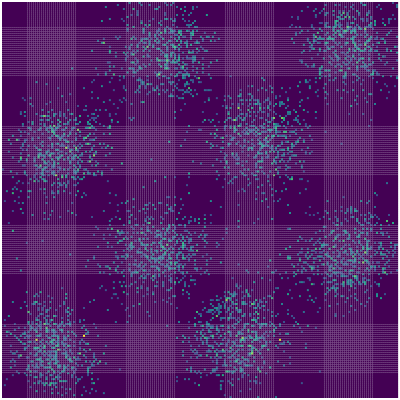}\\
            Target data & \includegraphics[width=0.3\textwidth]{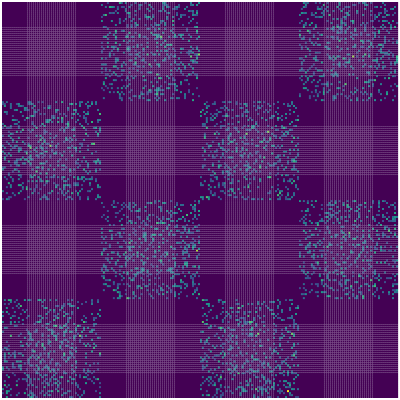} & \includegraphics[width=0.3\textwidth]{Figures_append/gaussian2checkerboard_png/fig_eight2eight_1_target.png} & \includegraphics[width=0.3\textwidth]{Figures_append/gaussian2checkerboard_png/fig_eight2eight_1_target.png} & \includegraphics[width=0.3\textwidth]{Figures_append/gaussian2checkerboard_png/fig_eight2eight_1_target.png}\\
            \bottomrule
            \toprule
          \end{tabular}
          }
        \end{table}
        
        \newpage
        \begin{table}[htb]
          \caption{Generated target samples from the trained CFM and the SFM at $t$ = 1 for different training iterations on transporting the $2$-d Gaussian mixture (8 modes) to the checkerboard.}
          \vspace*{1em}
          \label{tab:gaussian2checkerboard_itr}
          \resizebox{\linewidth}{!}{        \centering
          \begin{tabular}{m{3.5cm}<{\centering}|m{5cm}<{\centering}m{5cm}<{\centering}m{5cm}<{\centering}m{5cm}<{\centering}m{5cm}<{\centering}}
            \toprule
            \midrule
            Method & I-CFM or I-SFM (one2one) & I-SFM (eight2eight) & OT-CFM or OT-SFM (one2one) & OT-SFM (eight2eight)\\
            \midrule
            Iteration=0.5k & \includegraphics[width=0.3\textwidth]{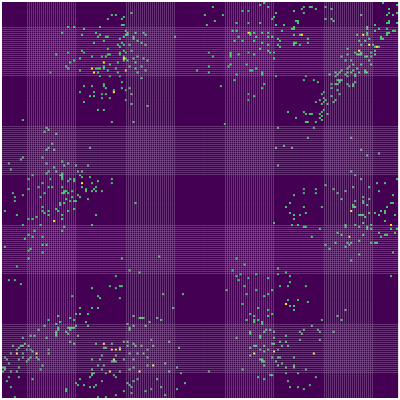} & \includegraphics[width=0.3\textwidth]{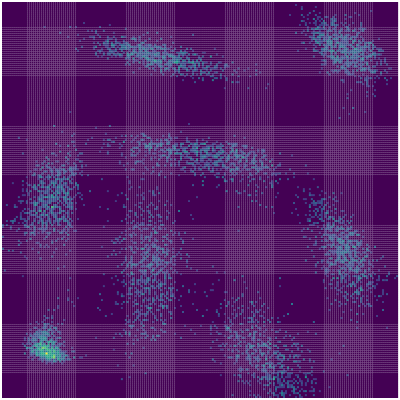} & \includegraphics[width=0.3\textwidth]{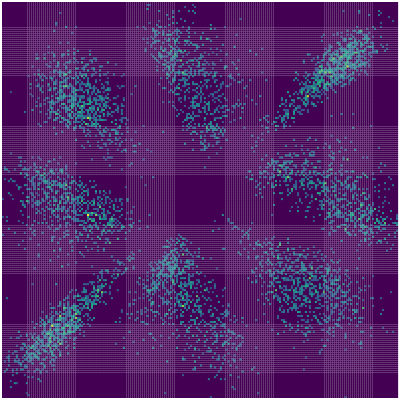} & \includegraphics[width=0.3\textwidth]{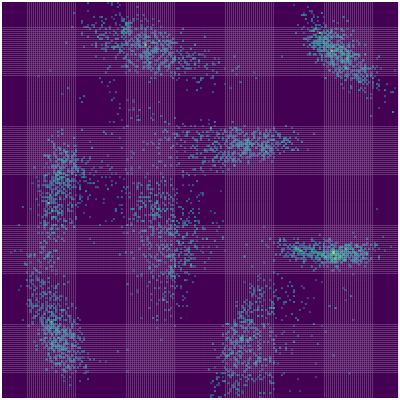}\\
            Iteration=1k & \includegraphics[width=0.3\textwidth]{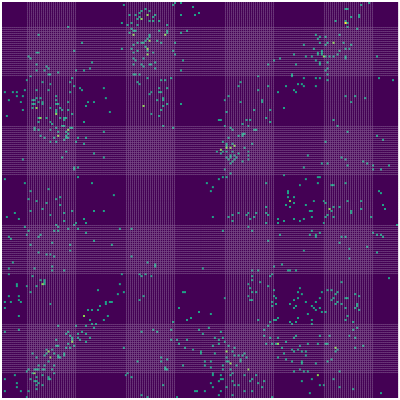} & \includegraphics[width=0.3\textwidth]{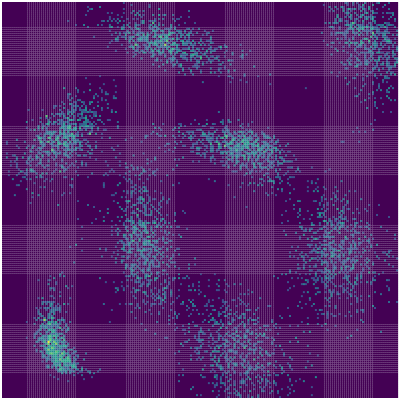} & \includegraphics[width=0.3\textwidth]{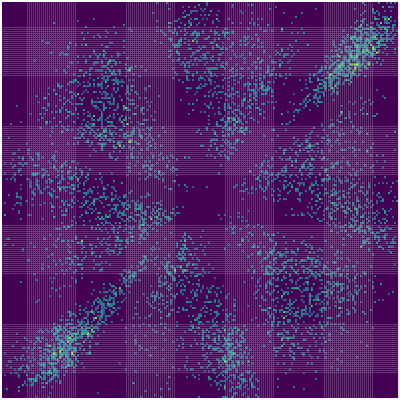} & \includegraphics[width=0.3\textwidth]{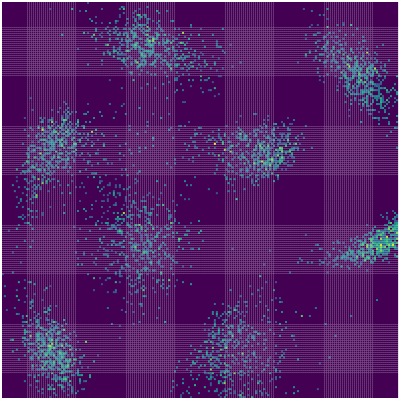}\\
            Iteration=2.5k & \includegraphics[width=0.3\textwidth]{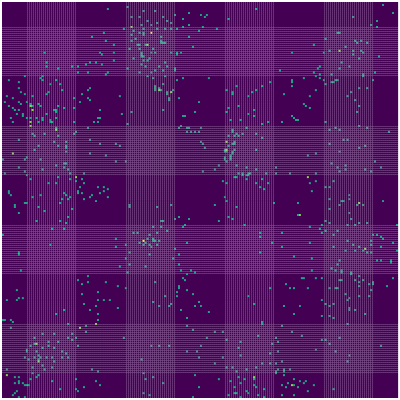} & \includegraphics[width=0.3\textwidth]{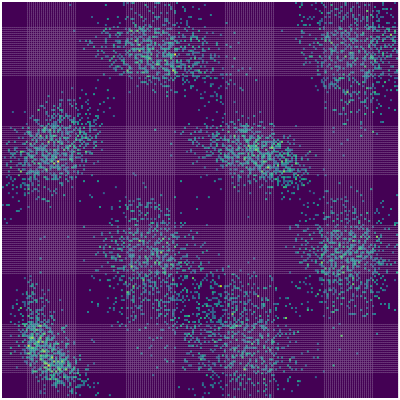} & \includegraphics[width=0.3\textwidth]{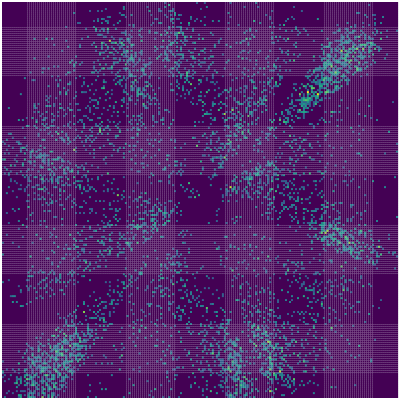} & \includegraphics[width=0.3\textwidth]{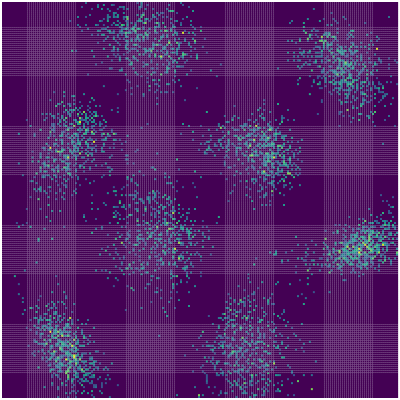}\\
            Iteration=5k & \includegraphics[width=0.3\textwidth]{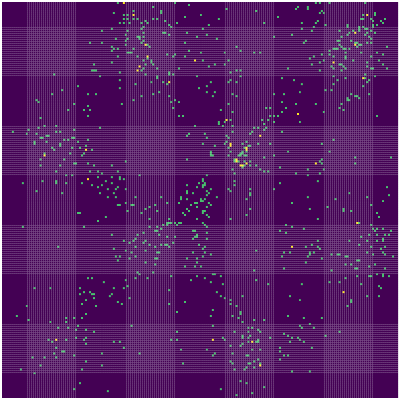} & \includegraphics[width=0.3\textwidth]{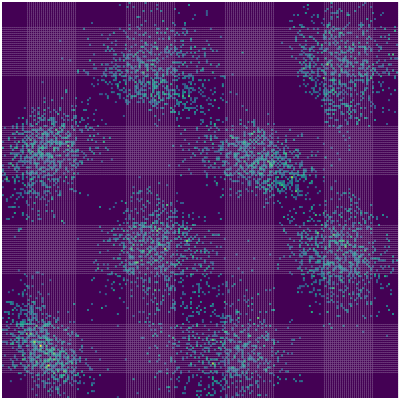} & \includegraphics[width=0.3\textwidth]{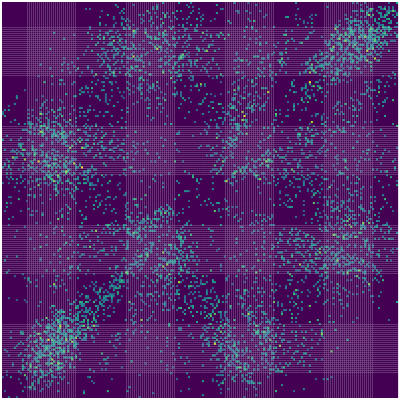} & \includegraphics[width=0.3\textwidth]{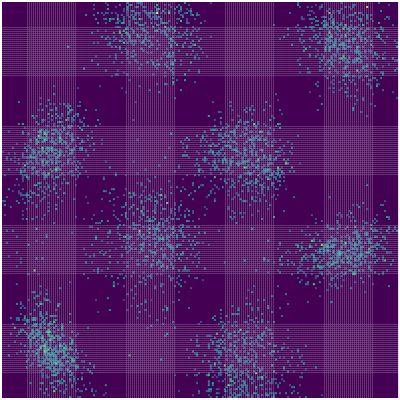}\\
            Iteration=10k & \includegraphics[width=0.3\textwidth]{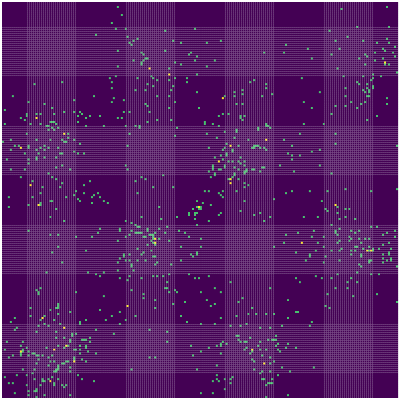} & \includegraphics[width=0.3\textwidth]{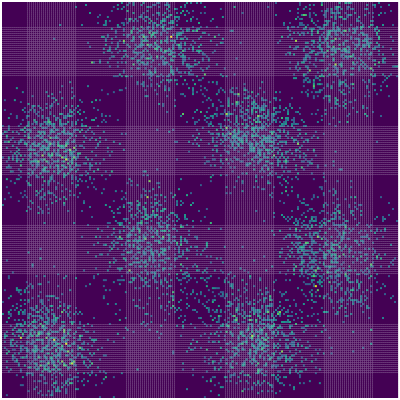} & \includegraphics[width=0.3\textwidth]{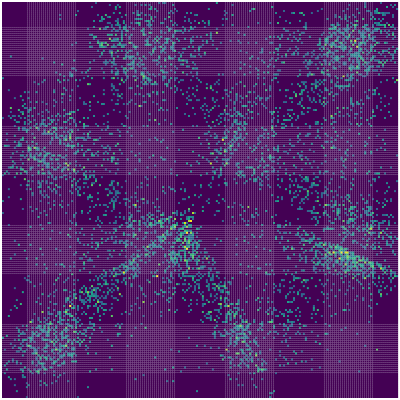} & \includegraphics[width=0.3\textwidth]{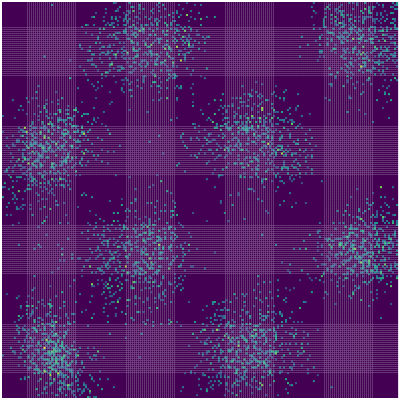}\\
            Iteration=20k & \includegraphics[width=0.3\textwidth]{Figures_append/gaussian2checkerboard_png/fig_one2one_0_100_itr_40.png} & \includegraphics[width=0.3\textwidth]{Figures_append/gaussian2checkerboard_png/fig_eight2eight_0_100_itr_40.png} & \includegraphics[width=0.3\textwidth]{Figures_append/gaussian2checkerboard_png/fig_one2one_1_100_itr_40.png} & \includegraphics[width=0.3\textwidth]{Figures_append/gaussian2checkerboard_png/fig_eight2eight_1_100_itr_40.png}\\
            \bottomrule
            \toprule
          \end{tabular}
          }
        \end{table}
        
        \newpage
        \begin{figure*}[htb]
          \begin{center}
            \centerline{\includegraphics[width=0.66\textwidth]{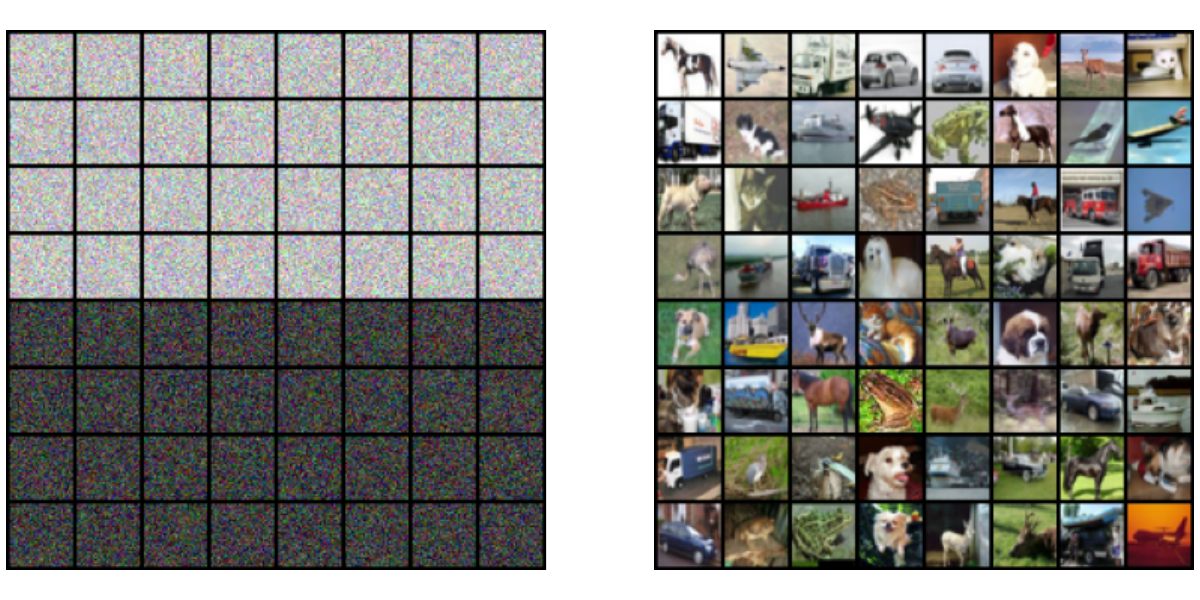}}
            % \vskip -0.1in
            \caption{True samples from the source distribution (left, Gaussian mixture) and the target distribution (right, CIFAR-10 dataset).}
            \label{fig_true_samples}
          \end{center}
          % \vskip -0.2in
        \end{figure*}
        
        \newpage
        \begin{table}[htb]
          \caption{Generated target samples from the trained I-CFM and the I-SFM at $t$ = 1 for different NFE on transporting the Gaussian mixture ($2$ modes) to the CIFAR-10 image dataset.}
          \vspace*{1em}
          \label{tab:cifar10_samples_I}
          \resizebox{\linewidth}{!}{        \centering
          \begin{tabular}{m{2.0cm}<{\centering}|m{5cm}<{\centering}m{5cm}<{\centering}m{5cm}<{\centering}m{5cm}<{\centering}m{5cm}<{\centering}}
            \toprule
            \midrule
            Method & NFE=10 & NFE=20 & NFE=40 & NFE=Adaptive\\
            \midrule
            I-SFM (one2one) & \includegraphics[width=0.3\textwidth]{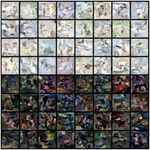} & \includegraphics[width=0.3\textwidth]{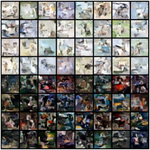} & \includegraphics[width=0.3\textwidth]{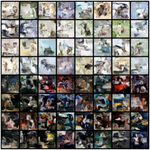} & \includegraphics[width=0.3\textwidth]{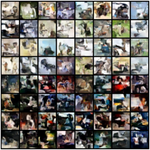}\\
            I-SFM (one2ten) & \includegraphics[width=0.3\textwidth]{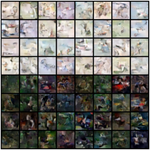} & \includegraphics[width=0.3\textwidth]{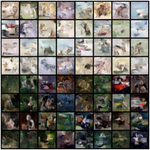} & \includegraphics[width=0.3\textwidth]{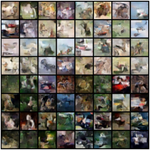} & \includegraphics[width=0.3\textwidth]{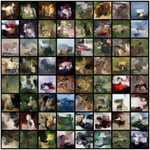}\\
            I-SFM (two2one) & \includegraphics[width=0.3\textwidth]{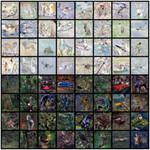} & \includegraphics[width=0.3\textwidth]{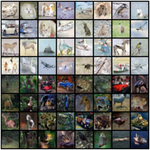} & \includegraphics[width=0.3\textwidth]{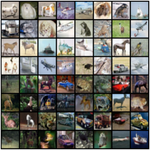} & \includegraphics[width=0.3\textwidth]{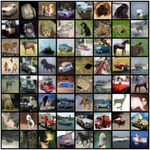}\\
            I-SFM (two2ten, mixed) & \includegraphics[width=0.3\textwidth]{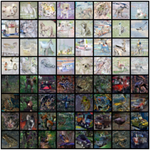} & \includegraphics[width=0.3\textwidth]{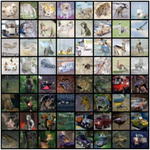} & \includegraphics[width=0.3\textwidth]{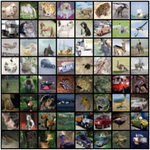} & \includegraphics[width=0.3\textwidth]{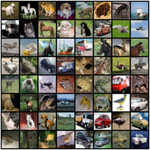}\\
            I-SFM (two2ten, extremal) & \includegraphics[width=0.3\textwidth]{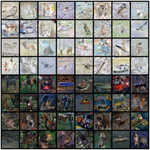} & \includegraphics[width=0.3\textwidth]{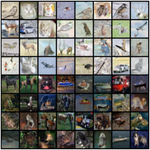} & \includegraphics[width=0.3\textwidth]{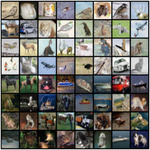} & \includegraphics[width=0.3\textwidth]{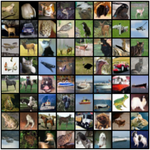}\\
            \bottomrule
            \toprule
          \end{tabular}
          }
        \end{table}
        
        \newpage
        \begin{table}[htb]
          \caption{Generated target samples from the trained OT-CFM and the OT-SFM at $t$ = 1 for different NFE on transporting the Gaussian mixture ($2$ modes) to the CIFAR-10 image dataset.}
          \vspace*{1em}
          \label{tab:cifar10_samples_OT}
          \resizebox{\linewidth}{!}{        \centering
          \begin{tabular}{m{2.0cm}<{\centering}|m{5cm}<{\centering}m{5cm}<{\centering}m{5cm}<{\centering}m{5cm}<{\centering}m{5cm}<{\centering}}
            \toprule
            \midrule
            Method & NFE=10 & NFE=20 & NFE=40 & NFE=Adaptive\\
            \midrule
            OT-SFM (one2one) & \includegraphics[width=0.3\textwidth]{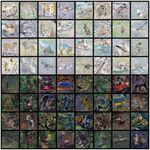} & \includegraphics[width=0.3\textwidth]{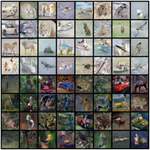} & \includegraphics[width=0.3\textwidth]{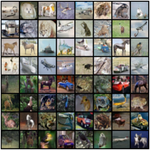} & \includegraphics[width=0.3\textwidth]{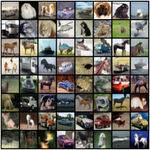}\\
            OT-SFM (one2ten) & \includegraphics[width=0.3\textwidth]{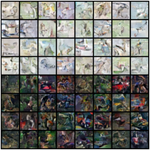} & \includegraphics[width=0.3\textwidth]{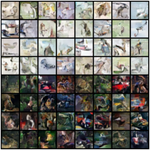} & \includegraphics[width=0.3\textwidth]{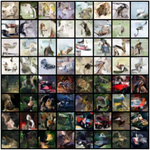} & \includegraphics[width=0.3\textwidth]{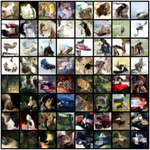}\\
            OT-SFM (two2one) & \includegraphics[width=0.3\textwidth]{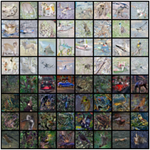} & \includegraphics[width=0.3\textwidth]{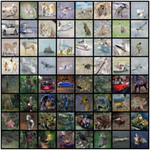} & \includegraphics[width=0.3\textwidth]{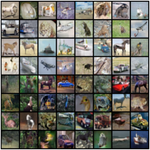} & \includegraphics[width=0.3\textwidth]{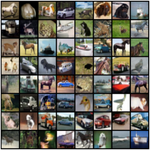}\\
            OT-SFM (two2ten, mixed) & \includegraphics[width=0.3\textwidth]{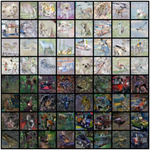} & \includegraphics[width=0.3\textwidth]{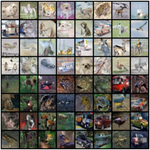} & \includegraphics[width=0.3\textwidth]{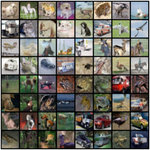} & \includegraphics[width=0.3\textwidth]{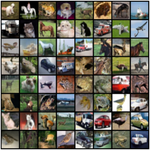}\\
            OT-SFM (two2ten, extremal) & \includegraphics[width=0.3\textwidth]{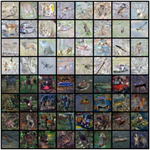} & \includegraphics[width=0.3\textwidth]{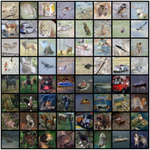} & \includegraphics[width=0.3\textwidth]{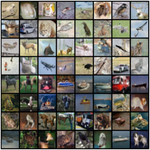} & \includegraphics[width=0.3\textwidth]{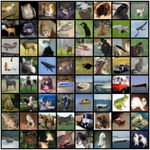}\\
            \bottomrule
            \toprule
          \end{tabular}
          }
        \end{table}
        
        \newpage
        \section{Experimental details}
        \label{sec:appd_experiment}
        It should be noted that our experimental implementations are heavily adapted from the open source code \href{https://github.com/atong01/conditional-flow-matching}{https://github.com/atong01/conditional-flow-matching} provided in \citet{tong2023conditional,tong2023improving}. All our experiments were conducted on a single 11GB GTX 1080 Ti GPU.
        
        For the synthetic experiments, we provide the detailed setup for different datasets (see Table~\ref{tab:synthetic_setup}).
        \begin{table}[htb]
          \caption{The detailed setup for different synthetic datasets.}
          \vspace*{1em}
          \label{tab:synthetic_setup}
          \resizebox{\linewidth}{!}{        \centering
          \begin{tabular}{c|c|c|c|c|c}
            \toprule
            \midrule
            \multicolumn{2}{c|}{Setup}  & Dataset $1$ & Dataset $2$ & Dataset $3$  & Dataset $4$\\
            \midrule
            \multirow{3}{*}{Data} & Dimension & $1$ & $1$ & $2$ & $2$ \\
            & $q_0$ & Gaussian mixture (2 modes) & Uniform distribution & $\mathcal{H}^1$ &  Gaussian mixture (8 modes)\\
            & $q_1$ & Gaussian mixture (2 modes) & Uniform mixture (2 modes) & $(1/2)\mathcal{H}^1$ & Checkerboard (8 squares)\\
            \midrule
            \multirow{5}{*}{Structure} & Hidden layer & $2$ & $2$ & $2$ & $2$ \\
            & Hidden neuron & $64$ & $64$ & $64$ & $64$ \\
            & Activation &  SELU   &  SELU  &  SELU &  SELU \\
            & Time input & True & True & True & True \\
            & Switching signal input & True (SFM) / False (CFM) & True (SFM) / False (CFM) & True (SFM) / False (CFM) & True (SFM) / False (CFM)\\
            \midrule
            \multirow{4}{*}{Training} & Batch size & $256$ & $256$ & $256$ & $256$\\
            & Iteration & $20k$ & $20k$ & $10k$ & $20k$ \\
            & Optimizer & Adam & Adam & Adam & Adam \\
            & Learning rate & $10^{-3}$ & $10^{-3}$ & $10^{-3}$ & $10^{-3}$ \\
            \midrule
            \multirow{1}{*}{Inference} & ODE solver & Euler/dopri5 & Euler/dopri5 & Euler/dopri5 & Euler/dopri5\\
            \midrule
            \multirow{1}{*}{List} &  Figures or tables & Figs.~\ref{fig_2gaussians}~\&~\ref{fig_2gaussians_joint_clustering}  & Figs.~\ref{fig_2uniform_shrink}~\&~\ref{fig_2uniform_shift} & Fig.~\ref{fig_Infinite_trained}~\&~Tab.~\ref{tab:infinite} & Tabs.~\ref{tab:gaussian2checkerboard_nfe}~\&~\ref{tab:gaussian2checkerboard_itr}\\
            \bottomrule
            \toprule
          \end{tabular}
          }
        \end{table}
        
        For the CIFAR-10 experiments, all methods used in our work were trained with the same setup as reported in \citet{tong2023conditional,tong2023improving}, only differences in the source distribution, the choice of probability path, and the switching mechanism for our proposed SFM and its variants. More precisely, we apply the UNet with the following structures and training configurations:
        \begin{itemize}
          \item Adam optimizer with $\beta_1 = 0.9$, $\beta_2 = 0.999$, $\epsilon=10^{-8}$, and no weight decay,
          \item channels $=128$,
          \item depth $=2$ ,
          \item channels multiple $=[1,2,2,2]$,
          \item heads $=4$,
          \item heads channels $=64$,
          \item attention resolution $=16$,
          \item dropout $=0.1$,
          \item batch size per gpu $= 128$, gpus $=1$,
          \item learning rate $=2 \times 10^{-4}$,
          \item gradient clipping with norm $=1.0$,
          \item exponential moving average weights with decay $=0.9999$.
        \end{itemize}
        For sampling, we use the traditional Euler integration or the adaptive step size solver \texttt{dopri5} from the torchdiffeq package. We use a batch size of $500$ for $100$ total batches for computing the Fr{\'e}chet Inception Distance (FID) through the TensorFlow-GAN library \url{https://github.com/tensorflow/gan}.

        Here, we provide the specific setup for the initial distribution of the CIFAR10 experiments, where $x_0$ has a $50\%$ chance of being $x_0 = torch.randn(3, 32, 32) / 4 + 0.5$ and a $50\%$ chance of being $x_0 = torch.randn(3, 32, 32) / 4 - 0.5$. The coupling matrices $P$ are set as 
        $$ \text{two2one}: \begin{pmatrix} 0.5 \\ 0.5 \end{pmatrix},$$ $$\text{one2ten}: \begin{pmatrix} 0.1 & 0.1 & \cdots & 0.1 \end{pmatrix},$$ 
        $$\text{two2ten, mixed}: \begin{pmatrix} 0.05 & 0.05 & \cdots & 0.05 \\ 0.05 & 0.05 & \cdots & 0.05 \end{pmatrix},$$ $$\text{two2ten, extremal}: \begin{pmatrix} 0.1 & 0.1 & 0.1 & 0.1 & 0.1 & 0 & 0 & 0 & 0 & 0 \\ 0 & 0 & 0 & 0 & 0 & 0.1 & 0.1 & 0.1 & 0.1 & 0.1 \end{pmatrix}.$$
        
      \end{document}